\documentclass[review,onefignum,onetabnum]{siamart251216}
\usepackage[utf8]{inputenc}
\usepackage{amsmath}
\usepackage{amssymb}
\usepackage{float}
\usepackage{multirow}
\usepackage{tikz}
\usepackage{csquotes}
\usepackage{empheq}
\usepackage{pgfplots}
\pgfplotsset{compat=1.18}
\usepackage{subcaption}
\usepackage{xcolor}
\usepackage{booktabs}
\usepackage{arydshln}
\usepackage{rotating} 
\usepackage{enumitem}

 \newtheorem{thm}{Theorem}[section]
 
 \newtheorem{lem}[thm]{Lemma}{\rm}

 {\rm}
 
 \newtheorem{rem}[thm]{Remark}
 
\numberwithin{equation}{section}

\def\x{\mathbf{x}}
\def\y{\mathbf{y}}
\def\z{\mathbf{z}}
\def\q{\mathbf{q}}
\def\vecalph{\mathbf{\alpha}}

\def\R{\mathbb{R}}
\def\N{\mathbb{N}}
\def\M{\mathbf{M}}

\def\B{\mathbf{B}}
\def\om{\mathbf{\Omega}}

\newcommand{\mysum}[2]{\underset{#1}{\overset{#2}{\sum}}}

\title{
Scalable anomaly detection via a univariate Christoffel function}

\author{ Florian Grivet\textsuperscript{1,2}
\and Didier Henrion\textsuperscript{2,3}
\and Jean-Bernard Lasserre\textsuperscript{2}
\and Louise Trav\'e-Massuy\`es\textsuperscript{2}}

\footnotetext[1]{Centre national d'\'etudes spatiales (CNES), Paris, France. \tt florian.grivet@cnes.fr}

\footnotetext[2]{LAAS-CNRS, University of Toulouse, CNRS, Toulouse, France.}

\footnotetext[3]{Faculty of Electrical Engineering, Czech Technical University in Prague, Czechia.}

\nolinenumbers

\begin{document}
\maketitle 

\begin{abstract}

Anomaly detection plays a critical role in identifying unusual patterns across domains such as fraud detection, network intrusion, and system fault diagnosis. Recently, Christoffel function-based methods, rooted in polynomial optimization, have emerged as promising alternatives to deep learning due to their strong mathematical foundations and computational frugality.
However, their practical applicability is hindered by the need to invert a matrix whose size grows exponentially with the data dimension, rendering the method intractable even for moderate-dimensional datasets. 
This paper addresses the dimensionality limitations of Christoffel function-based anomaly detection while preserving its key theoretical properties, i.e., the on-off support dichotomy behavior and the accurate support shape capture.
We introduce UCF, a univariate Christoffel function
which is based on the squared distance between the query point and the support points.
Extensive experiments on the ADBench benchmark demonstrate that UCF consistently outperforms 14 state-of-the-art baselines in terms of Average Precision. 
By resolving the scalability bottleneck of the Christoffel Function, this work expands the toolkit of anomaly detection methods with a robust, theoretically grounded, and universally applicable approach.

\end{abstract}

\begin{keywords}
    Anomaly Detection, Outlier Detection, Christoffel Function,
    Data mining, Polynomial Optimization.
\end{keywords}

\begin{AMS}
    68T05, 62-07, 62G07, 33C50
\end{AMS}


\section{Introduction}
Anomaly detection (also known as outlier detection) stands out as a crucial problem of data science and particularly in statistical inference, especially in view of new challenges with high-dimensional data as outlined in e.g. \cite{high-dim-stat}. It involves identifying patterns in data that do not conform to expected behavior. These non-conforming patterns, or \emph{anomalies}, are often indicative of critical and actionable information in various domains, such as fraud detection, network security, and fault detection in industrial systems. The importance of anomaly detection lies in its ability to preemptively identify issues that could lead to significant consequences if left unaddressed. In several contexts, highly desirable features of a detection method are frugality and explainability. Such features exclude (black-box) deep learning methods which require substantial tuning and computational resources. 

Therefore candidate methods fall under the category of shallow learning\footnote{As opposed to deep learning}, which can be categorized into two main categories: distance-based and density-based, and further divided into several types: proximity-based (like KNN~\cite{knn}), clustering-based (like CBLOF~\cite{cblof}), statistical-based (KDE~\cite{kde}, HBOS~\cite{hbos}, ECOD~\cite{ecod}, GMM~\cite{gmm}), projection-based (OCSVM~\cite{ocsvm}, LODA~\cite{loda}, PCA~\cite{pca}), and tree-based (like IForest~\cite{iforest})
methods. 
Their principles vary. For instance, clustering-based methods group data points into clusters based on similarity, with anomalies being points that do not belong to any cluster or belong to sparse clusters. KNN-based methods identify anomalies as points that are far from their nearest neighbors, indicating low density regions. 
Statistical methods model the data distribution and identify anomalies as points with low probability under the estimated distribution.

\newpage
Let us also mention the work \cite{Chazal} where the authors formalize a notion of \emph{distance function to a measure} as an improvement for geometric inference to classical distance functions to a compact set (especially in the presence of outliers).
In the discrete empirical setting with a cloud of $N$ data points, the resulting distance evaluated at $\x$, is the average distance between $\x$ and its $k$-nearest neighbors.

Only relatively recently, it has been advocated that the 
Christoffel-Darboux kernel and the Christoffel function (CF), classical tools from orthogonal polynomials and approximation theory, can be very useful in data analysis and mining \cite{sos,Lasserre2019,Lasserreetal2022} and could provide an additional tool in the arsenal of methods alluded to above. CF-based methods fall under the category of statistical methods. The associated anomaly score 
is obtained via evaluation of an explicit polynomial 
which depends on the statistical moments of the empirical measure defined by the set of samples at hand, hence capturing the statistics of the dataset. Among the other methods mentioned above, KDE methods are the most closely related. However, CF introduces a distinct perspective compared to KDE, as it leverages properties of orthogonal polynomials and moments to provide a robust and theoretically grounded approach not only to anomaly detection, but also to support inference and density approximation.

In particular, Ducharlet et al.~\cite{Ducharlet2025} have exploited and developed the use of CF for anomaly detection. Their work demonstrates the value of this approach in terms of performance, frugality (only one parameter, the CF-degree, to tune with the DyCF-method and none with the DyCG-method), and explainability.

Although the performance of CF-based methods is well-established, this multidimensional approach unfortunately remains limited to problems of low dimension as it requires to invert a matrix of size ${n+d \choose d}$ where $d$ is the data dimension and $n$ the degree of the CF. This limitation has been partially addressed by Askari et al.~\cite{askari2018} who proposed a kernel-based CF approach interpreted as a \emph{ridge regression} problem. As its interesting and crucial feature, the bottleneck is now
identified by the size $N$ of the sample and not its dimension $d$, as one has to invert a resulting matrix of size $N+1$, as opposed to ${n+d\choose d}$ for the standard degree-$n$ CF. In addition, the degree $n$ only affects the entries of the matrix, not its size and
in the empirical version, features other than polynomials are also allowed. Essentially the associated score function (which depends on a regularization parameter to tune) evaluates how ``far" is the new data $\x$ to the vector subspace of the features space generated by the data. However the size $N$ of the sample is still a serious penalization for the approach 
as an $O(N^3)$ matrix inversion is required. In a different direction, Billet et al.~\cite{billet2026} have also proposed an autoencoder architecture whose training loss includes the CF score, reducing the input's high dimension to a reasonable size in the latent space, which can then be reused as input for a CF-based anomaly detection method. Finally, for the class
of measures with conditional product structure described via a graphical-model (which represents interactions between variables), the authors in \cite{Lass-Slot} introduce a frugal Christoffel-like function whose crucial complexity parameter is now the size of the largest clique in the graph. 

\subsection*{Contribution} 
In this paper, we propose a new CF-based method that is conceptually aligned in spirit, yet fundamentally different at the technical level from \cite{Ducharlet2025}.

\paragraph{A simpler approach to high-dimensional anomaly detection}

To address the challenging problem of outlier detection in high-dimensional (and possibly large-scale) data, we propose a new scoring method. To evaluate a target point, instead of computing \emph{once and for all} a complex, high-dimensional multivariate Christoffel Function (CF), one introduces UCF, a univariate Christoffel function that \emph{depends} on the point $\x$ to evaluate. UCF considers the squared distance between our target point $\x$ and all the other data points. 
For each point $\x\in\R^d$, one introduces a mapping of the complex dataset onto a simple positive number, the average squared distance between our target point $\x$ and all points of the dataset. This in turn yields a \emph{one-dimensional} \emph{pushforward measure} that we call $\nu_\x$. Then our new anomaly score UCF for $\x$ is simply the 
Christoffel function associated with $\nu_\x$, and evaluated at zero.
It is important to note that we are \emph{not} just crushing high-dimensional data into one dimension, which could result in a loss 
of key information. Because the univariate measure $\nu_\x$ depends on $\x$, 
the UCF score which is recalculated at each new query point $\x$, is based on the average distance from $\x$
to the points of the dataset, and therefore the spatial relationships are preserved. 

\paragraph{A much lighter computational load}

The relevance of this new score is how little computing power it requires. To calculate a degree-$n$ score, we only need two steps: 

$\bullet$ Build a small matrix: one constructs a Hankel matrix of size $n+1$, the degree-$n$ moment matrix associated with the univariate measure $\nu_\x$. 
Its entries are calculated simply by averaging powers of the squared distances between the target point $\x$ and the rest of the dataset.

$\bullet$ Matrix inversion: one then inverts this small moment matrix and
the top-left value of the inverse provides the reciprocal of the score.

Because $n$ is usually a very small number (typically between 6 and 8), computing the inverse is fast and straightforward. This is a significant improvement over the traditional multivariate method, which requires inverting a large matrix of size $O(n^d)$, a task that is computationally heavy and prone to numerical errors, even in modest dimension $d$. While we do have to rebuild our small matrix every time we check a new point $\x$, the math is so simple that it remains highly efficient, unless the training dataset is extremely large.

\paragraph{Frugal and easy to tune}

To actually flag outliers, we plug our new UCF score into existing anomaly detection frameworks (e.g. such as the one of \cite{Ducharlet2025}), swapping out the heavy multivariate CF for our lightweight UCF. Just like the original method in small dimension, our UCF approach 
is highly frugal, i.e. it runs efficiently and requires tweaking only a single hyperparameter. 

\paragraph{Theoretical bonus: direct density approximation}

Beyond simply finding anomalies in data mining, the UCF approach has a distinguishing elegant theoretical feature which is interesting in its own right.
Normally, in order to approximate the true underlying probability density function using the traditional multivariate CF, one has to know the equilibrium measure (from pluripotential theory \cite{Baran,bedford,klimek}) of its support, which is rarely known (except for some specific geometries). 
However, with the UCF, the approach greatly simplifies. If the density function is continuous, it can actually be approximated directly just by looking at the limit of the UCF as the degree $n$ grows. Because one bypasses the unknown (multivariate) equilibrium measure, it provides us with a much cleaner, and more explicit relationship to the 
underlying density.

\color{black}

\paragraph{Theoretical advantages confirmed by benchmarking}

Extensive benchmarking was conducted to evaluate the UCF approach against recent baseline algorithms in the same category, using the 47 ADBench datasets —which features significant variations in sample size and dimensionality \cite{adbench}. The results confirm the advantages and interest of the UCF approach.

\subsection*{Structure}
This article is organized as follows. Section \ref{sec:background} introduces the CF, emphasizing its key theoretical properties and its role in anomaly detection. Section \ref{sec:ucf} presents our novel UCF, detailing its formulation, theoretical properties, and its application to anomaly detection. We also establish its connection to the original measure density, illustrated through a theoretical example involving the arcsine measure. Section \ref{sec:experiments} describes the comprehensive experiments conducted to evaluate the performance of UCF. We first introduce the 47 datasets and 14 baselines used, followed by the experimental setup designed to ensure reproducibility, and conclude with the key results. Finally, section \ref{sec:conclusion} summarizes the main findings, discusses limitations, and outlines directions for future research.

\section{Background: The Christoffel Function}
\label{sec:background}

The Christoffel-Darboux kernel and the Christoffel Function (CF) are classical tools that originate from the theory of approximation and orthogonal polynomials. Only relatively recently it was advocated 
\cite{sos,Lasserre2019, Lasserreetal2022,Ducharlet2025} that 
basic theoretical properties of the CF could be leveraged in data mining and data analysis 
via the empirical version of the CF associated with a given cloud of data points. In this section one recall some key properties of both the theoretical CF, referred to as the population CF, and its empirical counterpart.

\subsection{The population Christoffel Function}
Let $\x = \begin{pmatrix} x_1, x_2, \cdots, x_d \end{pmatrix} \in \R^d$. To define polynomials, we adopt the multi-index notation $\vecalph = \left( \alpha_i \right)_{i=1...d} \in \N^d$, such that the monomial $\x^\vecalph$ of total degree $|\vecalph| = \sum_{i=1}^{d} \alpha_i$ is given by $\x^\vecalph = x_1^{\alpha_1} x_2^{\alpha_2} \cdots x_d^{\alpha_d}$. In short form, we denote the set of $d$-variate polynomials by $\R[\x]$. The dimension of $\R_n[\x]$, the space of $d$-variate polynomials of degree at most $n$, is given by $s_d(n) = \begin{pmatrix} d+n \\ n\end{pmatrix}$. 

Let $\{P_i : 1 \leq i \leq s_d(n) \}$ be a basis of $\R_n[\x]$ and introduce:
\begin{align*}
    v_n : \R^d & \longrightarrow \R^{s_d(n)} \\
    \x & \longmapsto \left(P_1(\x), P_2(\x), \cdots, P_{s_d(n)}(\x) \right)^T\,.
\end{align*}
The monomials in $v_n(\x)$ are graded in the lexicographic order\footnote{lexicographic order: monomials are first ordered according to ascending total degree $|\vecalph|$, and then using lexicographic order on variables considering $\x_1=a, \x_2=b$, etc.}. For a polynomial
$Q\in \R[\x]$, write 
\[\x\mapsto Q(\x)\,=\,\langle \q,v_n(\x)\rangle\,,\quad\forall\x\in\R^d\,,\]
where $\q\in\R^{s_d(n)}$ is the coefficient vector of $Q$ in the basis $v_n$.

Let $\om \subset \R^d$ be a compact set with non-empty interior and let $\mu$ be a non-negative Borel measure whose support is $\om$.

\begin{definition}[Moment matrix]
\label{def:moment_matrix}
The moment matrix of order $n \in \N$, associated with measure $\mu$, and denoted by $\M_n(\mu) \in \R^{s_d(n) \times s_d(n)}$, is defined as
\begin{equation}
    \label{eq:mm}
    \M_n(\mu) = \int_{\R^d} v_n(\x) ~ v_n(\x)^T d\mu(\x).
\end{equation}
\end{definition}

Note that this matrix is symmetric positive definite, thus non-singular for all $n$ (see \cite[Section 2.2]{sos} or \cite[Remark 2.3]{vu2020rateconvergencegeometricinference} for the proof).

\begin{definition}[The Christoffel-Darboux Kernel]
\label{def:cd_kernel}
The CD kernel associated with the measure $\mu$, denoted by $K^\mu_n(\x, \y)$, is defined by:
\begin{equation}
    \label{eq:cdk}
    (\x, \y) \mapsto K^\mu_n(\x, \y) = \sum_{j=1}^{s_d(n)}T_j(\x)\,T_j(\y) \,,\quad\forall \x,\y\in\R^d\,,
    \end{equation}
    where $(T_j)_{j\in\N}$ is an arbitrary family of polynomials, orthonormal with respect to $\mu$.
\end{definition}
Importantly,  the CD kernel does \emph{not} depend on the family $(T_j)_{j\in\N}$ of orthonormal polynomials, and it turns out that
\[K^\mu_n(\x,\y)\,=\,v_n(\x)^T\M_n(\mu)^{-1}v_n(\y)\,,\quad\forall \x,\y\in\R^d\,.\]
One also introduces the polynomial
\begin{equation}
    \label{eq:Q}
    \x \mapsto Q_{\mu, n}(\x) = K^\mu_n(\x, \x) = v_n(\x)^T \M_n(\mu)^{-1} v_n(\x).
\end{equation}
which is the diagonal of the CD-kernel, and 
by \eqref{eq:cdk} is a sum-of-squares (SOS) polynomial of degree $2n$.
\vspace{0.5em}

\begin{definition}[The population Christoffel Function]
\label{def:population_christoffel}
The population CF of degree $n \in \N$, associated with the measure $\mu$, denoted by $\Lambda^\mu_n(\x)$, is defined as
\begin{equation}
    \label{eq:cf_int_p}
    \Lambda^\mu_n(\x) = \underset{Q \in \R_n[\x]}{min}\left\{ \int_\om Q(\z)^2 ~ d\mu(\z), \quad Q(\x) = 1 \right\}\,,\quad\forall \x\in\R^d\,.
\end{equation}
\end{definition}

Equivalently:
\begin{equation}
    \label{eq:cf_mm}
    \Lambda^\mu_n(\x) = \underset{\q\in \R^{s_d(n)}}{min}\left\{ \q^T \M_n(\mu) ~ \q, \quad \q^T v_n(\x) = 1 \right\}\,,
\end{equation}
which is a well-defined convex quadratic program. It also turns out that
\begin{equation}
    \label{equiv-def}
\Lambda^\mu_n(\x)^{-1}\,=\,K^\mu_n(\x,\x)\,=\,Q_{\mu,n}(\x)\,,\quad\forall\x\in\R^d\,.
\end{equation}

An interesting property of $Q_{\mu,n}$ is the dichotomy of its behavior with $n$, inside and outside the support $\om$ of $\mu$. Indeed by \cite[Lemma 4.3.1, Lemma 4.3.2]{Lasserreetal2022}, for fixed $\x\in\R^d$, the growth of $Q_{\mu,n}(\x)$ with $n$, is  at least exponential whenever $\x\not\in\om$,
whereas it is at most polynomial whenever $\x\in\mathrm{int}(\om)$.
This property proved to be crucial in using the CF 
as a score function for detecting outliers in data mining applications as in e.g. 
\cite{askari2018,Ducharlet2025,sos,Lasserre2019} (where in  such a context, the measure $\mu$ is simply the empirical measure supported on the dataset).

Next we state an important asymptotic property of the 
CF $\Lambda^\mu_n$ associated with a \emph{univariate} probability measure $\mu$ on a compact  interval $[a,b]$ of the real line. This result 
\cite[Theorem 3.3.1]{Lasserreetal2022} due to
Mat\'e, Nevai and Totik \cite{mate-nevai-totik},  will be crucial for approximating the density of $\mu$ in the multivariate case, and is re-stated below for the case $[a,b]=[-1,1]$.

\begin{thm}(Mat\'e-Nevai-Totik)
\label{thm1}
 Let $\mu$ be a positive measure on $[-1,1]$   with density $u(t)=d\mu(t)/dt$. If $\mu$ belongs to the Szeg\"o class, that is,
 \[\int_{-1}^1 \frac{\log u(t)}{\sqrt{1-t^2}}dt\,>\,-\infty\,,\]
 then the CF $\Lambda^\mu_n$ satisfies
 \begin{equation}
     \label{thm1-1}
     \lim_{n\to\infty}n\,\Lambda^\mu_n(t)\,=\,\pi\, u(t)\,\sqrt{1-t^2}\,,\quad\mbox{a.e. on [-1,1]}\,.
 \end{equation}
\end{thm}
The convergence is only pointwise almost everywhere on $[-1,1]$. To obtain stronger convergence 
properties (e.g. uniform convergence on compact subsets of $[-1,1]$) one has to assume additional regularity properties of the density.

\subsection{The empirical Christoffel Function}

~Let $\mathcal{X}$ be a cloud of $N$ data points $\x \in \R^d$ sampled from a  theoretical probability distribution $\mu$ with support $\om$. In practical applications of data mining, the underlying measure $\mu$ is unknown and so one introduces the discrete empirical measure 
$\mu_N := \frac{1}{N} \sum_{\x \in \mathcal{X}} \delta_{\x}$ supported on $\mathcal{X}$
(where $\delta_\x$ stands for the Dirac measure at $\x$).
The empirical version of the moment matrix reads:
\begin{equation}
    \M_n(\mu_N) = \frac{1}{N} \mysum{\x \in \mathcal{X}}{} v_n(\x) ~ v_n(\x)^T\,,\quad n\in\N\,.
    \label{eq:empirical_mm}
\end{equation}
If the size $N$ of $\mathcal{X}$ is greater than $s_d(n)$ and no polynomial of degree 
$n$ vanishes on $\mathcal{X}$, then $\M_n(\mu)$ is non singular \cite[Corollary 6.3.5]{Lasserreetal2022}.

\begin{definition}[Empirical Christoffel Function]
\label{def:empirical_cf}
    If $N:=|\mathcal{X}| \geq s_d(n)$ and no polynomial of degree $n$ vanishes on $\mathcal{X}$, the empirical CF is defined as
    \begin{equation}
        \label{eq:empirical_cf} 
        \Lambda_n^{\mu_N}(\x) = \frac{1}{v_n(\x)^T \M_n(\mu_N)^{-1} ~ v_n(\x)}.    
    \end{equation}
\end{definition}

According to \cite[Theorem 3.13]{Lasserre2019}, for every $n\in\N$, fixed,
the degree-$n$ empirical CF converges to the degree-$n$ population CF as $N$ increases:
\vspace{-0.2em}
\begin{equation*}
    \| \Lambda_n^{\mu_N} - \Lambda_n^{\mu} \|_\infty = \underset{\x \in \R^d}{\sup} \left \{|\Lambda_n^{\mu_N}(\x) - \Lambda_n^{\mu}(\x)| \right \} \underset{N \to \infty}{\longrightarrow} 0 \quad a.s.
\end{equation*}

\subsection{The Christoffel Function for outlier detection}
As already mentioned, a crucial feature of the degree-$n$ CF $\Lambda^\mu_n$ is the dichotomy of its behavior with $n$, depending on whether it is evaluated at a point inside or outside  the support $\Omega$.
As a result, sublevel sets of the polynomial $Q_{\mu, n}(\x)$
effectively capture the shape of the underlying dataset. 
As for fixed $n$, $\Lambda_n^{\mu_N}(\x)^{-1}$ converges to $\Lambda_n^{\mu}(\x)^{-1}$ as $N$ grows, these properties are preserved for finite datasets. Of course, when $n$ grows, 
the size $N$ of the sample needs to be adjusted accordingly as indicated in \cite{Lasserreetal2022}; however, in practice 
$n$ is kept relatively small (e.g. $n\leq 8,10$) and so a sufficiently large sample size $N$ 
is fine.

Consequently, $\Lambda^{\mu_N}_n(\x)$ is a well-suited \emph{scoring function} for outlier detection. Indeed one can define an appropriate level set with threshold $\gamma_{n, d}$ such that all points $\x \in \R^d$ verifying $\Lambda^{\mu_N}_n(\x)^{-1}>\gamma_{n,d}$ are considered as outliers. This defines the scoring function
\begin{equation}
    \label{eq:scoring_function}
    S_{n,d}(\x) = \frac{\Lambda^{\mu_N}_n(\x)^{-1}}{\gamma_{n,d}} = \frac{v_n(\x)^T \M_n(\mu_N)^{-1} ~ v_n(\x)}{\gamma_{n,d}}\,,\quad \x\in\R^d\,,
\end{equation}
and a point is detected as an outlier if $S_{n,d} \geq 1$.

While the CF provides a natural and efficient score function (even with moderate degree $n$) with 
strong links to the support of the underlying measure, its 
computation does not scale well. Indeed for fixed dimension $d$, 
the size of the moment matrix $\M_n(\mu)$ (or $\M_n(\mu^N)$) is ${d+n\choose d}$
(i.e. $O(n^d)$) which rapidly becomes an obstacle even for moderate dimension $d$ and degree $n$, and prevents from its use in many data mining applications. In the next section we describe how
to overcome this obstacle.

\section{A Univariate Christoffel Function approach}
\label{sec:ucf}
To address the high-dimensional data challenge, we propose a new score function to evaluate whether a point $\x$ is an outlier (or an anomaly).
This score function, namely the Univariate Christoffel Function (UCF), is 
the CF associated with the pushforward measure $\nu_\x$ of the empirical 
measure $\mu_N$ (supported on the initial sample) by the mapping $\y\mapsto \Vert\y-\x\Vert^2$, where $\x$
is the data point of interest. Hence $\nu_\x$ is a \emph{univariate}  probability  measure on the positive half line. 

At this stage it is important to realize that one does \emph{not} replace a high-dimensional problem with a one-dimensional one, which would be highly questionable as an unavoidable loss of information would occur. Indeed the univariate measure $\nu_\x$
\emph{depends} on $\x$ and hence changes with $\x$, whereas (the multivariate) $\Lambda^\mu_n$ is computed once and for all.
Thus the
univariate ``trick" is to be paralleled with the univariate \emph{needle polynomial} of Kro\'o and Lubinsky~\cite{kroo2013christoffel}
that was used to provide an upper bound on the (multivariate) CF $\Lambda^\mu_n(\x)$ 
at $\x$, by using a univariate polynomial of the variable $t:=\Vert\y-\x\Vert^2$. So the score function
is now $\Lambda^{\nu_\x}_n(0)$, i.e., the univariate CF associated with $\nu_\x$, evaluated at $z=0$. 

The computational burden for computing this new degree-$n$ score function $\Lambda^{\nu_\x}_n(0)$ is as follows:
\begin{itemize}
        \item Compute the (Hankel) moment matrix $\M_n(\nu_\x)$ of size $n+1$ via
        \begin{eqnarray}
                    \nonumber
                    \M_n(\nu_\x)[k,\ell]&:=&\int \Vert \y-\x\Vert^{2(k+\ell-2)}d\mu^N(\y)\,,\quad 1\leq k,\ell\leq n+1\\
        \label{formula}
        &=&\frac{1}{N}\sum_{i=1}^N(\Vert \y(i)-\x\Vert^{2})^{k+\ell-2}\,,\quad 1\leq k,\ell\leq n+1\,,
        \end{eqnarray}
        where $(\y(i))_{1\leq i\leq N}\subset\R^d$ is the sample of data points under investigation.
        \item Invert $\M_n(\nu_\x)$
        of size $n+1$ 
        to obtain the score
        \[1/\Lambda^{\nu_\x}_n(0)\,=\,e_0^T\M_n(\nu_\x)^{-1}e_0\,,\] where $e_0\in\R^{n+1}$ is the vector $(1,0,\ldots,0)$.
\end{itemize}
The crucial step of inverting $\M_n(\nu_\x)$ is 
quite straightforward, especially as one usually needs to do it for relatively modest degree $n$, say $n\leq 6, 8$. On the other hand, even in modest dimension $d$, inverting $\M_n(\mu_N)$ 
can be quite challenging even for relatively small $n$ (let alone numerical issues). As shown later in the paper, computing the entries of $\M_n(\nu_\x)$ in \eqref{formula} can be done quite efficiently, which is important as one needs to redo the computation when $\x$ changes.
\begin{rem}
  For fixed degree $n$, and in view of \eqref{formula}, the entry $(k,\ell)$ of 
  $\M_n(\nu_\x)$ is an explicit polynomial in $\x$ of degree $2(k+\ell-2)$. Therefore 
  $\Lambda^{\nu_\x}_n(0)$ is an explicit rational function of $\x$ which in principle could be computed once and for all (as the ratio $p_1(\x)/p_2(\x)$ of two determinants of $\M_n(\nu_\x)$). However this is practical only if $n$ is  small (e.g. with $n=2$, $\mathrm{deg}(p_i)=12$, $i=1,2$) and the dimension $d$ is modest to avoid storing a large vector of coefficients.
   \end{rem}

\subsection{The UCF score function}

Let $\om\subset\R^d$ be compact with nonempty interior and let $\om$ be the closure of its interior.  Let $d\mu(\x)=f(\x)d\x$ be a probability measure with 
support $\mathrm{supp}(\mu)=\om$, and with density $f>0$ on $\om$.
Given $\x\in\R^d$, fixed, arbitrary, let $g_\x:\R^d\to \R_+$
be the mapping $\y\mapsto g_\x(\y):=\Vert \y-\x\Vert^2$. 
Introduce the pushforward $\nu_\x:=g_\x\#\mu $ of $\mu$ by the mapping $g_\x$. It is the (univariate) measure on $\R_+$ defined by
\[\nu_\x(B)\,:=\,\mu(g_\x^{-1}(B))\,=\,\mu(\{\y \in \Omega : g_\x(\y) \in B\}),\,\quad\forall B\in\mathcal{B}(\R)\,.\]
Observe that 
\begin{equation} \label{supp_nu}
\mathrm{supp}(\nu_\x)\,\subset\,\left[\min_{y\in\om} g_\x(\y),\,\max_{y\in\om} g_\x(\y)\right]\,=:\,[a_\x,b_\x]\,\subset [0,\infty)\,,    
\end{equation}
and $a_\x=0$ whenever $\x\in\om$. More precisely, if $\om$ is connected then
$\mathrm{supp}(\nu_\x)=[a_\x,b_\x]$ because $g_\x$ is continuous. In addition,
\[\nu_\x(\{0\})\,:=\,g_\x^{-1}(\{0\})\,=\,\mu(\{\x\})\,=\,0\,,\]
as we have assumed that $\mu$ has no atom in its support. Next, 
observe that
\begin{eqnarray}
\nonumber
\Lambda_{2n}^\mu(\x)&=&\min_{p\in\R[\x]_{2n}}\,\{\,\int p^2\,d\mu:\:p(\x)\,=\,1\,\}\,,\quad\forall \x\in\R^d\\
\label{lien-1}
&\leq&\min_{q\in\R[t]_n}\,\{\,\int q(\Vert\y-\x\Vert^2)^2\,d\mu(\y):\:q(0)\,=\,1\,\}\,,\quad\forall \x\in\R^d\\
\label{line-2}
&=&\min_{q\in\R[t]_n}\,\{\,\int q(z)^2\,d\nu_\x(z):\:q(0)\,=\,1\,\}\,,\quad\forall \x\in\R^d\\
\label{lien-2}
&=&\Lambda^{\nu_\x}_n(0)\,,\quad\forall \x\in\R^d\,.
\end{eqnarray}

The inequality \eqref{lien-1} can be established as follows. Since \(g_{\x}\) has degree \(2\) in \(\y\), and \(q\) has degree at most \(n\), the composition \(p_q(\y):=q(\|\y-\x\|^2)\)
is a polynomial in \(\y\) of degree at most \(2n\). Moreover, \(
p_q(\x)=q(\|\x-\x\|^2)=q(0)\). Hence, if \(q(0)=1\), then
\(p_q(\x)=1\). Therefore every feasible polynomial \(q\) for the univariate constrained problem in \eqref{line-2} yields a feasible polynomial \(p_q\) for the multivariate constrained problem.

The equality \eqref{lien-2} follows from the definition of the push-forward measure. Let \(q \in \mathbb{R}[t]_n\) and \(h(t):=q(t)^2\). Then \(h\) is Borel measurable and bounded on the compact support of \(\nu_{\x}\), hence integrable.
By the defining property of the pushforward measure,
\(\int h(z)\,d\nu_{\x}(z)=\int h(g_{\x}(\y))\,d\mu(\y)\).
Substituting \(h(z)=q(z)^2\) and \(g_{\x}(\y)=\|\y-\x\|^2\), one obtains
\(\int q(z)^2\,d\nu_{\x}(z)=\int q(\|\y-\x\|^2)^2\,d\mu(\y)\).

So for every $\x\in\R^d$, the CF $\Lambda^\mu_{2n}(\x)$ 
is bounded above by 
$\Lambda^{\nu_\x}_n(0)$ in \eqref{lien-2}, where the latter upper bound is obtained 
via the univariate CF $\Lambda^{\nu_\x}_n$ evaluated at $0$. Importantly, this univariate CF $\Lambda^{\nu_\x}_n$ \emph{depends} on $\x\in\R^d$.  

\centerline{\emph{So for every $\x\in\R^d$, the new UCF score function is simply $\Lambda^{\nu_\x}_n(0)$.}}

To detect whether $\x\in\R^d$ is an outlier, we then use the strategy of \cite{Ducharlet2025} but now with the UCF $\Lambda^{\nu_\x}_n(0)$ instead of $\Lambda^{\mu_N}_n(\x)$. Like the original DyCF method, the UCF-based method maintains a focus on frugality, requiring very little tuning, actually only one hyperparameter. 

\subsection{Dichotomy property}
In this section we show that for every fixed $\x\in\R^d$,
the growth of $\Lambda^{\nu_\x}_n(0)$ with $n$
has the same nice and desirable dichotomy property as
$\Lambda^\mu_n(\x)$, namely at most polynomial inside $\om$ and at least exponential outside.
\begin{lem}
\label{lem0}
Let $\om\subset\R^d$ be compact with nonempty interior and let $\om$ be the closure of its interior.  
Let $\mu$ be a probability measure with $\mathrm{supp}(\mu)=\om$, and assume that $\mu$ has a density w.r.t. Lebesgue measure, bounded from below in $\mathrm{int}(\om)$.

(i) If $\x\in\mathrm{int}(\om)$ then  as $n$ increases,
the growth of $1/\Lambda^{\nu_\x}_n(0)$ is at most polynomial in $n$. 

(ii) If $\x\not\in\om$ then as $n$ increases,
the growth of $1/\Lambda^{\nu_\x}_n(0)$ is at least exponential in $n$. 
\end{lem}
\begin{proof}
 (i) It is enough to observe that by \eqref{lien-1}-\eqref{lien-2},
\[1/\Lambda^{\nu_\x}_n(0)\,\leq\,1/\Lambda^\mu_{2n}(\x)=O(n^d)\,,\quad\forall \x\in\mathrm{int}(\om)\,,\]
and the last equality follows from e.g. 
\cite[Lemma 4.3.2]{Lasserreetal2022}.

Next to prove (ii) consider the univariate polynomial
\[z\mapsto 
q_n(z):=\left(1-\frac{z}{b_\x}\right)^n \in \R[z]_n.
\]
If satisfies $q_n(0)=1$, and for any $z\in[a_\x,b_\x]$,
	\[
	0 \le \frac{z}{b_\x} \le 1
	\quad\mathbb \Rightarrow\quad
	1 - \frac{z}{b_\x} \le 1 - \frac{a_\x}{b_\x} =: c_\x < 1\,,
	\]
    so that
	\[
		|q_n(z)| \le c_\x^n
		\qquad\text{for all }z\in[a_\x,b_\x]\,,
	\]	    
    and therefore
    \[
    \Lambda^{\nu_\x}_n(0)\,\leq\,\int q_n(z)^2 d\nu_\x(z)\,\leq\,  \int c_\x ^{2n} d\nu_\x(z) = c_\x ^{2n}.
    \]	
    Write $c_\x^{2n} = \exp(-2n|\log c_\x|)$ and set
	$d_\x := 2|\log c_\x|>0$ to obtain the desired result
		\[
		\Lambda_n^{\nu_\x}(0)
		\le \exp(-d_\x n)\,.
		\]
    
\end{proof}

\subsection{Approximation of the density of $\mu$}
In addition to the properties derived in the previous sections, we claim that the use of the pushforward $\nu_\x$ is interesting in its own right in the classical setting of $(\om,\mu)$ where $\om$
is compact and $\mu=f(\x)d\x$ is a probability measure with $\mathrm{supp}§(\mu)=\om$. Indeed when $f$ is continuous, one shows that one may approximate $f(\x)$ pointwise on $\om$ from the limit 
$\lim_{n\to\infty} n\,\Lambda^{\nu_\x}_n(z)$ for fixed $z>0$ small enough so that $\B(\x,\sqrt{z})\subset\om$. 
This is in contrast with analogue results for $\Lambda^\mu_n(\x)$, which involve the  equilibrium measure of $\om$, in general not known explicitly.

Let us develop our proposal further below. We here assume that $\mu$ is absolutely continuous with respect to the Lebesgue measure on $\om$ with integrable density $f$, i.e. $\mu(d\y)=f(\y)d\y$ on $\om$, with $f\in L^1(\om)$ and $\int_\om f\,d\y=1$. 

One interesting property of the CF is to provide an asymptotic result in terms of the underlying density. Namely, under some regularity assumption on $(\om,\mu)$,
\begin{equation}
    \label{equlibrium}
    \lim_{n\to\infty}{n+d\choose d}\,\Lambda^\mu_n(\x)\,=\,f(\x)/\omega_E(\x)\,,\quad\forall \x\in\mathrm{int}(\om)\,,
    \end{equation}
where $\omega_E$ is the density of the equilibrium measure (in pluripotential theory \cite{Baran,bedford,klimek}) associated with $\om$. However, and unfortunately, except for special geometries like the Euclidean ball, the unit box, or the simplex (and their image by an affine transformation), $\omega_E$ is not known in general, which makes \eqref{equlibrium} of limited practical interest. 

In contrast, we claim that the proposed UCF $\Lambda^{\nu_\x}_n$ allows to obtain an approximation of $f(\x)$, as closely as desired. The reason is that 
the support of the univariate measure $\nu_\x$, 
being an interval $I:=[a_\x,b_\x]$ of the real line (e.g. if $\om$ is connected), its  associated equilibrium
density $\omega_I$ (in pluripotential theory) is \emph{known} and equal to $1/\pi\sqrt{(z-a_\x)(b_\x-z)}$.

So, for every $\tau>0$, let $\B(\x,\tau):=\{\y: \Vert \y-\x\Vert<\tau\}$. 
\begin{lem}
    \label{lem1}
    Let $0<f$ be continuous on $\om$, and let $\x\in\mathrm{int}(\om)$. 
    Then for $z$ small enough so that $\B(\x,\sqrt{z})\subset\om$, it holds 
    \begin{equation}
        \label{lem1-1}\nu_\x([0,z])\,=\,\int_{\B(\x,\sqrt{z})}f(\y)\,d\y\,,
    \end{equation}
    and the density $z\mapsto  q_\x(z)$ of $\nu_\x$ reads
    \begin{equation}
        \label{lem1-2}q_\x(z)\,=\,
        \frac{z^{-1/2}}{2}\,\int_{\Vert \y-\x\Vert=\sqrt{z}}f(\y)\,d\sigma(\y)\,,
        \end{equation}
    where $\sigma$ is surface measure on the sphere $\partial B(\x;\sqrt{z})=\{\y\in\R^d: \Vert \y-\x\Vert=\sqrt{z}\}$.
\end{lem}
\begin{proof}
It holds
\[\nu_\x([0,z])\,=\,\mu(g_\x^{-1}([0,z]))\,=\,\mu(\B(\x,\sqrt{z}))\,,\quad\mbox{as soon as $\B(\x,\sqrt{z})\subset\om$}\,,\]
and therefore,
\[\nu_\x([0,z])\,=\,\int_{\B(\x,\sqrt{z})}f(\y)\,d\y\,,\]
which is \eqref{lem1-1}. Next, to obtain \eqref{lem1-2} 
just take the derivative with respect to $z$.
\end{proof}
We next leverage the important asymptotic property of univariate measures on an interval, namely Theorem \ref{thm1}, to obtain an approximation of the density $f(\x)$ of $\mu$ (for $\x\in\mathrm{int}(\om)$) via that of $\nu_\x$ at $0$.

\begin{thm}
\label{thm2}
Let $f$ be continuous, strictly positive on $\om$,
and let $\x\in\mathrm{int}(\om)$. Then for all
$z>0$ small enough to ensure $\B(\x,\sqrt{z})\subset\om$,
    \begin{equation}
        \label{thm2-1}q_\x(z)\,\approx\,
        f(\x)\frac{\pi^{d/2}\,z^{d/2-1}}{\Gamma(d/2)}\,,
                        \end{equation}
                        and if $\om$ is connected, then for almost all $z$ small enough to ensure $\B(\x,\sqrt{z})\subset\om$,
\begin{equation}
        \label{thm2-2}
        \lim_{n\to\infty}n\,\Lambda^{\nu_\x}_n(z)
        \,=\,q_\x(z)\pi\,\sqrt{z(b_\x-z)}\,.\end{equation}
        In particular, from \eqref{thm2-2} and \eqref{thm2-1}, it holds
        \begin{equation}
                \label{thm2-3}
        \lim_{n\to\infty}n\,\Lambda^{\nu_\x}_n(z)\,\approx\,
        f(\x)\frac{\pi^{d/2+1}\,z^{(d-1)/2}\,\sqrt{b_\x-z}}{\Gamma(d/2)}\,.
        \end{equation}
            \end{thm}
\begin{proof}
Let $\x\in\mathrm{int}(\om)$. Eq \eqref{thm2-1} follows from
\eqref{lem1-2} and the integral form of the Mean Value Theorem. Next to get \eqref{thm2-2}-\eqref{thm2-3}
it suffices to prove that
the density $q_\x(z)$ belongs to the Szeg\"o class, i.e.,
\begin{equation}
\label{aux-szego}
    \int_0^{b_\x} \frac{\log q_\x(z)}{\sqrt{z\,(b_\x-z)}}dz\,>\,-\infty\,.\end{equation}
Indeed as $\x\in\mathrm{int}(\om)$ and $\om$ is connected, $\mathrm{supp}(\nu_\x)=[0,b_\x]$,
and so by Theorem \ref{thm1}:
\[\lim_{n\to\infty} n\,\Lambda^{\nu_\x}_n(z)\,=\,q_\x(z)\,\pi\,\sqrt{z(b_\x-z)}\,\quad\mbox{a.e. on $[0,b_\x]$,}\]
and therefore, for almost all $z$ small enough to ensure $\B(\x,\sqrt{z})\subset\om$, one obtains
\begin{eqnarray*}
    \lim_{n\to\infty} n\,\Lambda^{\nu_\x}_n(z)&=&q_\x(z)\,\pi\,\sqrt{z(b_\x-z)}\\
        &\approx&f(\x)\,\frac{\pi^{d/2+1}\,z^{(d-1)/2}\,\sqrt{b_\x-z}}{\Gamma(d/2)}\,,
        \end{eqnarray*}
        which is \eqref{thm2-3}.
        To prove \eqref{aux-szego} observe that as $f>0$ on $\om$ and $f$ being continuous, $f>\delta$ on $\om$ for some $\delta>0$. Therefore
        \[q_\x(z)\,>\,\delta\frac{z^{-1/2}}{2}\,\sigma(\partial\B(\x,\sqrt{z}))
        \,=\,\frac{\delta\,\pi^{d/2}}{\Gamma(d/2)}z^{d/2-1}\]
        and so
        \begin{eqnarray*}
            \int_0^{b_\x}\frac{\log q_\x(z)}{\sqrt{z(b_\x-z)}}dz&>&\frac{\delta\,\pi^{d/2}}{\Gamma(d/2)}\int_0^{b_\x}\frac{\log z^{d/2-1}}{\sqrt{z(b_\x-z)}}dz\\
        &=&\frac{\delta\,\pi^{d/2}\,(d/2-1)}{\Gamma(d/2)}\int_0^{b_\x}\frac{\log z}{\sqrt{z(b_\x-z)}}dz
        \end{eqnarray*}
        which is finite.
\end{proof}
So in view of \eqref{thm2-3}, notice that $n\,\Lambda^{\nu_\x}_n(z)$ behaves like $a\,z^{(d-1)/2}$ for small $z$ and large $n$ (for some constant $a$), a nice smooth behavior as soon as $d\geq 2$.

So the correct scaling to analyze the asymptotic behavior of $\Lambda^{\nu_\x}_n(z)$ with respect to the density $f$ of $\mu$ at a point $\x\in\mathrm{int}(\om)$ is:
\[\frac{\Gamma(d/2)}{\pi^{1+d/2}}\cdot\frac{n\,\Lambda^{\nu_\x}_n(z)}{z^{(d-1)/2}}\quad ( \approx f(\x)\,\sqrt{b_\x})\]
understood for very small $z>0$ and large degree $n$. In doing so one obtains the approximation
\[f(\x)\,\approx\,\frac{\Gamma(d/2)}{\sqrt{b_\x}\,\pi^{1+d/2}}\cdot\frac{n\,\Lambda^{\nu_\x}_n(z)}{z^{(d-1)/2}}\,.\]

\subsection{Example: the arcsine interval}
Let us use a simple one-dimensional example to
describe how the density of the univariate pushforward measure $\nu_\x$ looks like
(not only in a neighborhood of $0$), and to illustrate 
Theorem \ref{thm2}, i.e., emphasize that what matters to recover the density is
the asymptotic behavior of $n\Lambda^{\nu_x}_n(z)$ as $n\to\infty$ \emph{for small} $z>0$
and not for $z=0$.
For this purpose, we use the arcsine interval. 

Let \(d=1\), \(\om=[-1,1]\), and let
\[
d\mu(y)=\frac{1}{\pi\sqrt{1-y^2}}\,\mathbf 1_{(-1,1)}(y)\,dy\,,
\]
i.e., $\mu$ is the classical arcsine measure on \([-1,1]\) with density $f(y)=1/(\pi\sqrt{1-y^2})$. It is the canonical example
for which the scaled Christoffel function is asymptotically constant in the
interior of the support: for \(|x|<1\), one has
\[
( n+1)\Lambda_n^\mu(x)\longrightarrow \frac{f(x)}{\omega(x)}\,=\,1\,,
\]
uniformly on compact subsets of $(-1,1)$, where $\omega(x)$ is the density of the equilibrium measure of the interval $[-1,1]$ (and here $\omega=f$). At the endpoints,
\[
(n+1)\Lambda_n^\mu(\pm1)\longrightarrow \frac12.
\]

Consider the corresponding pushforward measure $\nu_x$. From its definition (cf. \eqref{supp_nu}), the support of $\nu_x$ reads:
\[
\operatorname{supp}(\nu_x)=
\begin{cases}
[0,(1+|x|)^2], & |x|\le 1,\\[1mm]
[(|x|-1)^2,(|x|+1)^2], & |x|>1,
\end{cases}
\]
and its density is
\[q_x(z)\,=\,
\frac{d\nu_x}{dz}(z)
=
\frac{1}{2\pi\sqrt z}
\sum_{\sigma\in\{\pm1\}}
\frac{\mathbf 1_{\{-1\le x+\sigma\sqrt z\le 1\}}}
{\sqrt{1-(x+\sigma\sqrt z)^2}},
\qquad z>0.
\]
In particular,
\[
q_0(z)=\frac{1}{\pi\sqrt{z(1-z)}}\,\mathbf 1_{(0,1)}(z)\,dz,
\]
and
\[
q_{\pm1}(z)
=
\frac{1}{2\pi z^{3/4}\sqrt{2-\sqrt z}}\,
\mathbf 1_{(0,4)}(z)\,dz.
\]



\begin{figure}[h!]
\centering

\begin{subfigure}[t]{0.48\textwidth}
\centering
\begin{tikzpicture}
\begin{axis}[
    width=\textwidth,
    height=6.3cm,
    xmin=0, xmax=2.4,
    ymin=0, ymax=3.0,
    axis lines=left,
    xlabel={$z$},
    ylabel={density},
    title={Case \(|x|=\frac12<1\)},
    xtick={0.25,2.25},
    xticklabels={$(1-|x|)^2$,$(1+|x|)^2$},
    ytick=\empty,
    clip=false
]

\addplot[
    black, ultra thick,
    domain=0.01:0.245,
    samples=350,
    restrict y to domain*=0:3
]
{(1/pi)*(1/sqrt(x))*(1/sqrt(3 - 4*sqrt(x) - 4*x) + 1/sqrt(3 + 4*sqrt(x) - 4*x))};

\addplot[
    black, ultra thick,
    domain=0.245:2.255,
    samples=350,
    restrict y to domain*=0:3
]
{1/(pi*sqrt(x)*sqrt(3 + 4*sqrt(x) - 4*x))};

\addplot[dashed] coordinates {(0.25,0) (0.25,3.0)};
\addplot[dashed] coordinates {(2.25,0) (2.25,3.0)};


\end{axis}
\end{tikzpicture}
\end{subfigure}
\hfill
\begin{subfigure}[t]{0.48\textwidth}
\centering
\begin{tikzpicture}
\begin{axis}[
    width=\textwidth,
    height=6.3cm,
    xmin=0, xmax=6.6,
    ymin=0, ymax=1.1,
    axis lines=left,
    xlabel={$z$},
    ylabel={density},
    title={Case \(|x|=\frac32>1\)},
    xtick={0.25,6.25},
    xticklabels={$(|x|-1)^2$,$(|x|+1)^2$},
    ytick=\empty,
    clip=false
]

\addplot[
    black, ultra thick,
    domain=0.27:6.26,
    samples=400,
    restrict y to domain*=0:1.1
]
{1/(pi*sqrt(x)*sqrt(-5 + 12*sqrt(x) - 4*x))};

\addplot[dashed] coordinates {(0.25,0) (0.25,1.1)};
\addplot[dashed] coordinates {(6.25,0) (6.25,1.1)};

\end{axis}
\end{tikzpicture}
\end{subfigure}

\caption{Densities $q_x(z)$ of the pushforward measure $\nu_x$ for the arcsine measure on \([-1,1]\). Left: the case \(|x|=\tfrac12<1\). Right: the case \(|x|=\tfrac32>1\).}
\label{fig:nu-x-density-arcsine}
\end{figure}
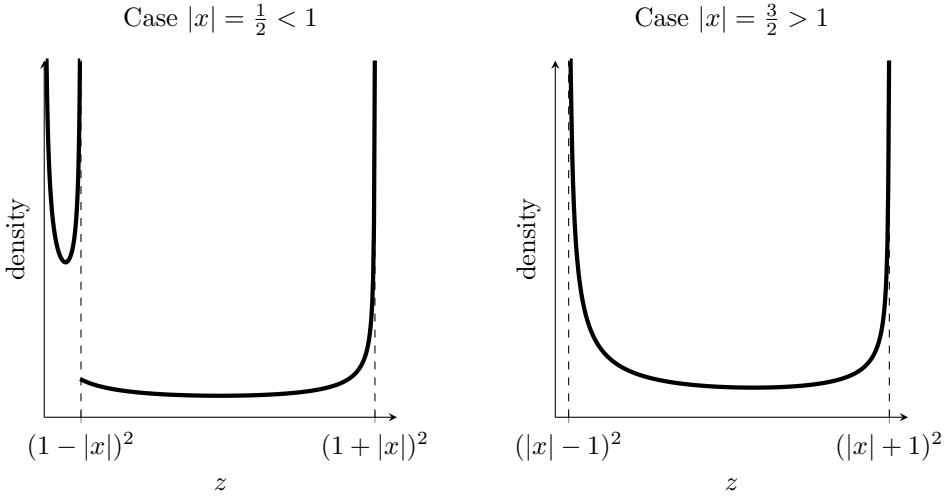

Figure~\ref{fig:nu-x-density-arcsine} displays the density of \(\nu_x\) in the two regimes
$x\in \mathrm{int}(\om)$ and $x\not\in\om$, 
for the representative values \(|x|=\tfrac12\) and \(|x|=\tfrac32\).
When \(0\le |x|<1\), the support starts at \(0\), and the density has an integrable hard-edge singularity at the origin of order \(z^{-1/2}\). In addition, at the interior point \(z=(1-|x|)^2\), one of the two preimages reaches the endpoint of \([-1,1]\), which creates a second square-root singularity. When \(|x|>1\), the support is separated from the origin, and the density is carried by the interval \([(|x|-1)^2,(|x|+1)^2]\); it then has square-root singularities at both endpoints of the support, corresponding to the two boundary points \(y=\pm1\) of the original arcsine measure.

Recall that $f(x)=1/(\pi\sqrt{1-x^2})$ whenever $x\in\mathrm{int}(\om)$.
Hence, with $x\in\mathrm{int}(\om)$ and very small $z$, observe that
\[q_x(z)\,\pi\sqrt{z\,(b_x-z)}\,\approx\frac{1}{2\pi\sqrt{z}}\times \frac{2}{\sqrt{1-x^2}}\times 
\pi\,\sqrt{z}\sqrt{b_x}\,=\,\frac{\sqrt{b_x}}{\sqrt{1-x^2}}\,.\]
which is the same as
\[f(x)\,\frac{\pi^{3/2}z^{(d-1)/2}\sqrt{b_x}}{\Gamma(1/2)}=f(x)\pi\sqrt{b_x}\,=\,\frac{\sqrt{b_x}}{\sqrt{1-x^2}}\,,\]
in accordance with \eqref{thm2-2}-\eqref{thm2-3} for small $z>0$. 

This example also illustrates that to recover asymptotic properties of the original measure $\mu$ (e.g. its density $f$) via the pushforward $\nu_x$ (for $x\in\mathrm{int}(\om)$), what matters
is the behavior of the latter at small $z>0$ and not at $z=0$.
Indeed $\Lambda^{\nu_0}_n(0)=1/(2n+1)$  and so
\[\frac{1}{2}\,=\,\lim_{n\to\infty}n\,\Lambda^{\nu_0}_n(0)\,\neq\,f(0)\,\frac{\pi^{3/2}\sqrt{b_x}}{\Gamma(1/2)}\,=\,1\,=\,\lim_{n\to\infty}n\,\Lambda^{\nu_0}_n(z)\quad\mbox{for small $z>0$.}
\]

\begin{figure}[t]
\centering
\includegraphics[width=0.78\textwidth]{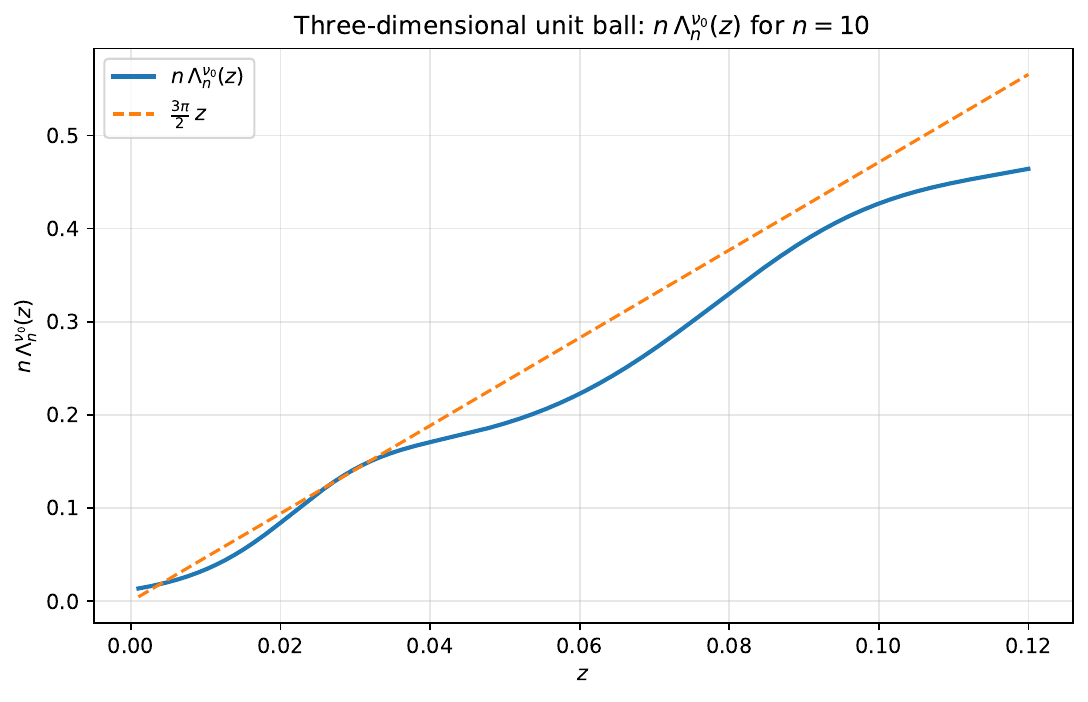}
\caption{The solid curve is $\Lambda_n^{\nu_0}(z)$ for small $z>0$ and  degree \(n=10\), where $\nu_0$ is the pushforward by \(y\mapsto \|y\|^2\) of the uniform measure on the unit ball in dimension $d=3$. The dashed line is the small-\(z\) asymptotic profile.}
\label{fig:ball3-small-z}
\end{figure}

The example of Figure~\ref{fig:ball3-small-z} illustrates the same mechanism in
a smoother setting. Let
\[
\Omega=B(0,1)\subset\mathbb R^3,
\qquad
d\mu(y)=\frac{3}{4\pi}\mathbf 1_{\{\|y\|\le 1\}}\,dy,
\qquad
\x=0.
\]
Then the pushforward by \(y\mapsto \|y\|^2\) has density
\[
q(z)=\frac32\sqrt z\,\mathbf 1_{[0,1]}(z).
\]
The solid curve is
\[
z\longmapsto n\,\Lambda_n^{\nu_0}(z),
\qquad z\in[10^{-3},1.2\cdot10^{-1}],
\]
for \(n=10\), and the dashed line is the small-\(z\) asymptotic profile
\[
\frac{3\pi}{2}\,z
=
f(0)\,\frac{\pi^{5/2}}{\Gamma(3/2)}\,z.
\] Since the pushforward density behaves like $q(z)\sim \frac32\sqrt z$ for small $z$,
the bulk asymptotic profile of \(n\,\Lambda_n^{\nu_0}(z)\) is linear at first
order:
\[
n\,\Lambda_n^{\nu_0}(z)\approx \frac{3\pi}{2}\,z.
\]
Thus, unlike the one- or two-dimensional cases, the three-dimensional profile is differentiable at the origin.
This makes the near-edge bulk behavior easier to visualize and supports the
general principle that, for \(d\ge 3\), the function
\[
f(x)\,\frac{\pi^{1+d/2}}{\Gamma(d/2)}\,\sqrt{b_x}\,z^{(d-1)/2}
\]
provides a smooth local model near \(z=0\).

\section{Experiments}
\label{sec:experiments}
To evaluate the performance of UCF, we conducted a comprehensive set of experiments within a shallow unsupervised anomaly detection framework. Our primary goal is to assess the effectiveness, robustness, and scalability of UCF across diverse datasets and baseline methods. For full reproducibility, the code and experimental configurations are available on GitHub\footnote{GitHub link: \url{https://github.com/fgrivet/ucf-scalable-ad-in-any-dimension}}.


\subsection{Datasets and Baselines}
\label{sec:datasets_algorithms}
Our evaluation leverages the ADBench benchmark suite \cite{adbench}, which comprises 47 datasets spanning a wide range of feature dimensions, sample sizes, and anomaly ratios. The properties of these datasets are summarized in \cref{fig:datasets_properties} and detailed in Appendix \ref{app:dataset_properties}, \cref{tab:datasets_properties}.

To better understand the strengths and weaknesses of UCF, we categorize the datasets based on two criteria:

\vspace{-0.5em}
\begin{minipage}[t]{0.61\textwidth}
    \begin{enumerate}[label=\arabic*., leftmargin=0pt, itemsep=0pt]
        \item \textbf{Category} (as in \cite{liscalable}):
        \begin{enumerate}[label=(\alph*), leftmargin=*, itemsep=0pt]
            \item \textit{Small:} $N \leq 1 000$ and $d \leq 50$.
            \item \textit{Medium:} $1000 < N \leq 10 000$ and $d \leq 50$;
            \item \textit{Large:} $N > 10 000$ and $d \leq 50$;
            \item \textit{High-dimensional:} $d > 50$;
        \end{enumerate}
    \end{enumerate}
\end{minipage}%
\begin{minipage}[t]{0.34\textwidth}
    \begin{enumerate}[label=\arabic*., leftmargin=0pt, start=2, itemsep=0pt]
        \item \textbf{Anomaly ratio}:
        \begin{enumerate}[label=(\alph*), leftmargin=*, itemsep=0pt]
            \item Less than 3\%;
            \item Between 3\% and 8\%;
            \item Between 8\% and 13\%;
            \item Between 13\% and 25\%; 
            \item Greater than 25\%
        \end{enumerate}
    \end{enumerate}
\end{minipage}

\begin{figure}[H]
    \centering
    \resizebox{0.6\textwidth}{!}{\begin{tikzpicture}
\begin{axis}[
    clip=false,
    ybar stacked,
    bar width=30pt,
    nodes near coords,
    legend style={at={(0.5, 1.10)}, fill=white, legend image post style={scale=0.8}, nodes={scale=0.6, transform shape,  at={(-1, 0.09)}}, anchor=north, legend columns=-1, /tikz/every even column/.append style={column sep=0.5cm},},
    xlabel={Category},
    ylabel={Count},
    symbolic x coords={Small, Medium, Large, High-dimensional},
    xtick=data,
    enlarge x limits=0.2,
    width=10cm,
    name=barplot,
]

\addplot+[ybar, fill=blue!50] plot coordinates {
    (Small, 1)
    (Medium, 4)
    (Large, 6)
    (High-dimensional, 3)
};
\addlegendentry{$< 3\%$}

\addplot+[ybar, fill=orange!50] plot coordinates {
    (Small, 4)
    (Medium, 4)
    (Large, 3)
    (High-dimensional, 2)
};
\addlegendentry{$3\% - 8\%$}

\addplot+[ybar, fill=green!50] plot coordinates {
    (Small, 2)
    (Medium, 2)
    (Large, 0)
    (High-dimensional, 2)
};
\addlegendentry{$8\% - 13\%$}

\addplot+[ybar, fill=red!50] plot coordinates {
    (Small, 2)
    (Medium, 2)
    (Large, 1)
    (High-dimensional, 1)
};
\addlegendentry{$13\% - 25\%$}

\addplot+[ybar, fill=yellow!50] plot coordinates {
    (Small, 3)
    (Medium, 3)
    (Large, 1)
    (High-dimensional, 1)
};
\addlegendentry{$> 25\%$}

\node[above] at (rel axis cs:0.5, 1.10) {\small{\textbf{Anomaly Ratio}}};

\end{axis}

\end{tikzpicture}}
    \caption{Number of datasets per category and anomaly ratio}
    \label{fig:datasets_properties}
\end{figure}
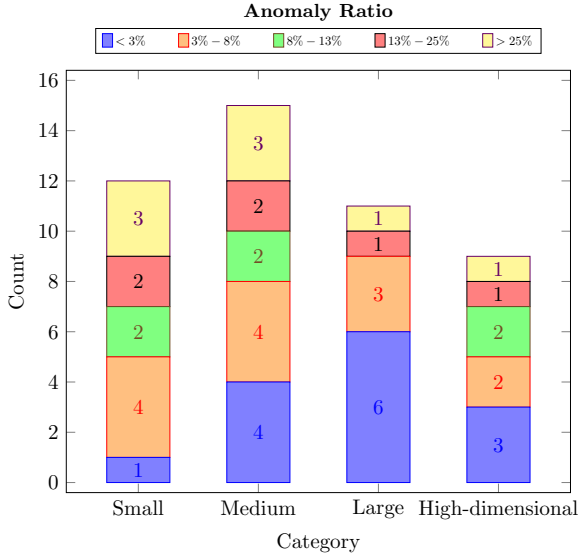

To provide a robust comparison, we evaluate the performance of UCF against 10 established baselines spanning a variety of anomaly detection paradigms: proximity-based (KNN), clustering-based (CBLOF), density-based (KDE, HBOS, ECOD, GMM, OCSVM), projection-based (LODA, PCA), and tree-based (IForest) methods. Additionally, we compare UCF with other Christoffel-based approaches where applicable, namely: DyCF \cite{Ducharlet2025}, DyCG \cite{Ducharlet2025}, and KernelCF \cite{askari2018}.

Below, we outline the core principles of each baseline.

\paragraph{KNN}
The $k$-Nearest Neighbors (KNN) algorithm \cite{knn} computes the anomaly score by measuring the distances (typically Euclidean, but other distances such as the Gini parametric \cite{gini} can also be used) to the $k$-th nearest neighbors of each point.

\paragraph{CBLOF}
The Cluster-Based Local Outlier Factor (CBLOF) \cite{cblof} first clusters the data (e.g. using K-means \cite{kmeans}), and then calculates the anomaly score of a sample as the product of its cluster size and its distance to the centroid of a so-called \textquote{large cluster}.

\paragraph{KDE}
Kernel Density Estimation (KDE) \cite{kde} estimates the probability density function of the data using a kernel function (e.g. Gaussian, exponential, or Epanechnikov).

\paragraph{HBOS}
Histogram-Based Outlier Score (HBOS) \cite{hbos} constructs univariate histograms with bins of equal-width for each dimension. The outlier score is the sum of the bin heights.

\paragraph{ECOD}
Empirical Cumulative Distribution Functions (ECOD) \cite{ecod} estimate the Empirical Cumulative Distribution Function (ECDF) of each variable separately. The outlier score is given by the sum of the minimum of the left and right tail probabilities.

\paragraph{GMM}
Gaussian Mixture Models (GMMs) \cite{gmm} establish clusters as a combination of Gaussian distributions, allowing the computation of the probability that each data point belongs to a particular cluster.

\paragraph{OCSVM}
The One-Class Support Vector Machine (OCSVM) \cite{ocsvm} uses the kernel trick to map data to a higher-dimensional space, and then finds the optimal hyperplane that maximizes the margin between normal data points and the origin. Anomalies are identified as data points far from this hyperplane.

\paragraph{LODA}
The Lightweight on-line detector of anomalies (Loda) \cite{loda} approximates the probability density of the input data projected onto a single projection vector.
The outlier score derived from the average log-probability across projection vectors.

\paragraph{PCA}
Principal Component Analysis (PCA) \cite{pca} reduces the dimensionality by selecting the eigenvectors of the covariance matrix with the highest eigenvalues as projection vectors. The outlier score is given by the sum of the weighted Euclidean distance between each sample and the hyperplane constructed by these selected eigenvectors.

\paragraph{IForest}
Isolation Forest (IForest) \cite{iforest} randomly partitions the input space, and its outlier score is given by the inverse number of splits necessary to isolate a given sample.

\paragraph{KernelCF}
Kernel-based Outlier Detection 
(KernelCF) \cite{askari2018} introduces a regularization parameter $\rho>0$ to ensure that $\M_n(\mu)+\rho\,\mathbf{I}$ is invertible. This is very useful for large dimension $d$ when 
the sample size $N$ is not large enough to ensure that $\M_n(\mu)$ is non singular (in which case the standard CF is not defined). When $\M_n(\mu)$ is invertible, the resulting regularized CF is a lower bound for the standard CF, and otherwise is interpreted as a ridge regression problem (see Eq. (17) in \cite{askari2018}).
By using an appropriate polynomial basis and leveraging kernel methods, the computational bottleneck is shifted from inverting a matrix whose size grows exponentially with the data dimension $d$ (when $\M_n(\mu)$ is invertible) to inverting a matrix whose size now scales with the size $N$ of the training sample
(independent of $d$). Moreover the degree $n$ only impacts the entries of this fixed size matrix.
While on the one hand this method introduces a regularization parameter $\rho$ to tune, on the other hand it can also use non polynomial kernels (but at the price of introducing new parameters to tune, e.g. the length scale $\sigma$ for RBF kernel).

\subsection{Experimental Setup}
To ensure a fair and reproducible comparison, we use the default implementations of all baseline algorithms from the PyOD library\footnote{\url{https://github.com/yzhao062/pyod}}. 
For UCF and DyCF, we compute degrees ranging from 2 to 8 (where feasible for DyCF), and report the result for the optimal degree\footnote{With respect to the Average Precision (AP) metric defined in \cref{sec:evaluation_metric}.} on each dataset. For DyCG, we use degrees from 2 to 5. KernelCF is configured with default parameters as recommended in \cite{askari2018}, employing both the linear (lin) and RBF (RBF) kernels.

For baseline methods, we standardize the data to zero mean and unit variance, as recommended in \cite{adbench}. For Christoffel-based methods (including UCF), we scale the data to the range [-1, 1] to ensure numerical stability in the computation of Chebyshev polynomials. 

Each experiment is repeated 5 times per dataset and algorithm, with results averaged across runs. For each repetition, we employ a stratified split with a fixed test ratio of 30\%. 

To guaranty consistency, we enforce a 2-hour time limit and 16 GB memory cap per run. All experiments were conducted on a Slurm HPC architecture equipped with AMD EPYC Milan 7713, 64C (2.0GHz-225W) processors.

\subsection{Evaluation metrics}
\label{sec:evaluation_metric}
We evaluate the performance of each algorithm using AP (Average Precision), and AUROC (Area Under the Receiver Operating Characteristic Curve), which respectively quantify the trade-off between precision and recall, and between the true positive rate (TPR) and false positive rate (FPR). Note that average precision is particularly informative for imbalanced datasets, as it focuses on the positive (anomaly) class and directly reflects the trade-off between precision and recall, whereas AUROC can be overly optimistic in such cases. 

We also report training time, inference time, and total execution time to assess computational efficiency.

Precision and recall are defined as:
\begin{equation*}
    \text{Precision} = \frac{TP}{TP + FP}, \quad \text{Recall} = \frac{TP}{TP + FN},
\end{equation*}

while TPR and FPR are given by:
\begin{equation*}
    \text{TPR} = \frac{TP}{TP + FN}, \quad \text{FPR} = \frac{FP}{FP + TN},
\end{equation*}

where TP, FP, TN, and FN denote true positives, false positives, true negatives, and false negatives, respectively.

\subsection{Results}
The averaged AP scores for the 5 repetitions and each dataset category are reported in \cref{tab:ap_performances_summary}, while results grouped by anomaly ratio are available in \cref{tab:ap_performances_summary_ar}. \cref{fig:boxplot_all_datasets} presents a boxplot of the average precision for each algorithm across all 47 ADBench datasets. \cref{tab:it_performances_summary} shows the inference time in seconds per sample for each category of dataset and each algorithm. Detailed per-dataset results for all methods and metrics are available in the supplemental materials. 

\begin{table}[h]
\centering
\caption{AP scores per datasets categories}
\vspace{-2em}
\caption*{\small Best result in \textbf{bold}, second-best \underline{underlined}.}
\resizebox{0.98\textwidth}{!}{%
\begin{tabular}{lccccc}
\toprule
\textbf{Algorithm} & \textbf{Small} & \textbf{Medium} & \textbf{Large} & \textbf{High-dimensional} & \textbf{All} \\
\midrule
\textbf{CBLOF} & 0.467 $\pm$ 0.312 & \underline{0.371 $\pm$ 0.227} & 0.237 $\pm$ 0.203 & \underline{0.344 $\pm$ 0.304} & 0.359 $\pm$ 0.266 \\
\textbf{DyCF} & 0.349 $\pm$ 0.287 & 0.275 $\pm$ 0.200 & 0.309 $\pm$ 0.265 & 0.049 $\pm$ 0.117 & 0.259 $\pm$ 0.247 \\
\textbf{DyCG} & 0.094 $\pm$ 0.135 & 0.043 $\pm$ 0.087 & 0.060 $\pm$ 0.108 & 0.000 $\pm$ 0.000 & 0.052 $\pm$ 0.101 \\
\textbf{ECOD} & 0.492 $\pm$ 0.297 & 0.333 $\pm$ 0.215 & 0.313 $\pm$ 0.263 & 0.256 $\pm$ 0.212 & 0.354 $\pm$ 0.256 \\
\textbf{GMM} & 0.445 $\pm$ 0.287 & 0.299 $\pm$ 0.167 & 0.303 $\pm$ 0.270 & 0.335 $\pm$ 0.268 & 0.344 $\pm$ 0.245 \\
\textbf{HBOS} & \underline{0.530 $\pm$ 0.282} & 0.328 $\pm$ 0.229 & 0.255 $\pm$ 0.295 & 0.317 $\pm$ 0.319 & 0.360 $\pm$ 0.287 \\
\textbf{IForest} & \textbf{0.537 $\pm$ 0.309} & 0.365 $\pm$ 0.252 & \underline{0.316 $\pm$ 0.357} & 0.302 $\pm$ 0.313 & \underline{0.386 $\pm$ 0.309} \\
\textbf{KDE} & 0.454 $\pm$ 0.310 & 0.327 $\pm$ 0.179 & 0.191 $\pm$ 0.282 & 0.155 $\pm$ 0.158 & 0.295 $\pm$ 0.260 \\
\textbf{KNN} & 0.507 $\pm$ 0.327 & 0.324 $\pm$ 0.159 & 0.173 $\pm$ 0.205 & 0.238 $\pm$ 0.154 & 0.319 $\pm$ 0.249 \\
\textbf{KernelCF (lin)} & 0.461 $\pm$ 0.276 & 0.306 $\pm$ 0.237 & 0.138 $\pm$ 0.272 & 0.217 $\pm$ 0.214 & 0.289 $\pm$ 0.271 \\
\textbf{KernelCF (RBF)} & 0.528 $\pm$ 0.327 & 0.337 $\pm$ 0.265 & 0.122 $\pm$ 0.237 & 0.289 $\pm$ 0.328 & 0.326 $\pm$ 0.314 \\
\textbf{LODA} & 0.456 $\pm$ 0.284 & 0.303 $\pm$ 0.231 & 0.192 $\pm$ 0.174 & 0.226 $\pm$ 0.285 & 0.301 $\pm$ 0.258 \\
\textbf{OCSVM} & 0.493 $\pm$ 0.310 & 0.352 $\pm$ 0.247 & 0.224 $\pm$ 0.300 & 0.331 $\pm$ 0.315 & 0.354 $\pm$ 0.296 \\
\textbf{PCA} & 0.517 $\pm$ 0.313 & 0.287 $\pm$ 0.235 & 0.307 $\pm$ 0.275 & 0.194 $\pm$ 0.339 & 0.333 $\pm$ 0.300 \\
\textbf{UCF} & 0.527 $\pm$ 0.316 & \textbf{0.413 $\pm$ 0.231} & \textbf{0.334 $\pm$ 0.310} & \textbf{0.361 $\pm$ 0.308} & \textbf{0.414 $\pm$ 0.288} \\
\midrule
\textbf{Average} & 0.457 $\pm$ 0.291 & 0.311 $\pm$ 0.211 & 0.232 $\pm$ 0.255 & 0.241 $\pm$ 0.242 & 0.316 $\pm$ 0.263 \\
\bottomrule
\end{tabular}
}%
\label{tab:ap_performances_summary}
\end{table}

\begin{table}[h]
\centering
\caption{AP scores per datasets anomaly ratio}
\vspace{-2em}
\caption*{\small Best result in \textbf{bold}, second-best \underline{underlined}.}
\resizebox{0.98\textwidth}{!}{%
\begin{tabular}{lcccccc}
\toprule
\textbf{Algorithm} & \textbf{$<$ 3\%} & \textbf{3\% - 8\%} & \textbf{8\% - 13\%} & \textbf{13\% - 25\%} & \textbf{$>$ 25\%} & \textbf{All} \\
\midrule
\textbf{CBLOF} & \underline{0.303 $\pm$ 0.263} & 0.318 $\pm$ 0.331 & 0.328 $\pm$ 0.151 & 0.285 $\pm$ 0.047 & 0.602 $\pm$ 0.221 & 0.359 $\pm$ 0.266 \\
\textbf{DyCF} & 0.205 $\pm$ 0.200 & 0.245 $\pm$ 0.254 & 0.215 $\pm$ 0.204 & 0.134 $\pm$ 0.150 & 0.501 $\pm$ 0.285 & 0.259 $\pm$ 0.247 \\
\textbf{DyCG} & 0.008 $\pm$ 0.015 & 0.031 $\pm$ 0.031 & 0.049 $\pm$ 0.054 & 0.050 $\pm$ 0.122 & 0.167 $\pm$ 0.182 & 0.052 $\pm$ 0.101 \\
\textbf{ECOD} & 0.264 $\pm$ 0.226 & 0.347 $\pm$ 0.335 & 0.352 $\pm$ 0.176 & 0.319 $\pm$ 0.159 & 0.549 $\pm$ 0.204 & 0.354 $\pm$ 0.256 \\
\textbf{GMM} & 0.219 $\pm$ 0.177 & 0.345 $\pm$ 0.305 & 0.355 $\pm$ 0.150 & 0.295 $\pm$ 0.069 & 0.589 $\pm$ 0.240 & 0.344 $\pm$ 0.245 \\
\textbf{HBOS} & 0.252 $\pm$ 0.274 & \underline{0.395 $\pm$ 0.389} & 0.298 $\pm$ 0.166 & 0.335 $\pm$ 0.118 & 0.559 $\pm$ 0.199 & 0.360 $\pm$ 0.287 \\
\textbf{IForest} & 0.286 $\pm$ 0.332 & 0.393 $\pm$ 0.396 & \textbf{0.371 $\pm$ 0.172} & \underline{0.335 $\pm$ 0.124} & 0.596 $\pm$ 0.222 & \underline{0.386 $\pm$ 0.309} \\
\textbf{KDE} & 0.134 $\pm$ 0.140 & 0.294 $\pm$ 0.304 & 0.306 $\pm$ 0.137 & 0.227 $\pm$ 0.088 & \underline{0.620 $\pm$ 0.233} & 0.295 $\pm$ 0.260 \\
\textbf{KNN} & 0.180 $\pm$ 0.174 & 0.303 $\pm$ 0.291 & 0.321 $\pm$ 0.120 & 0.278 $\pm$ 0.045 & 0.618 $\pm$ 0.229 & 0.319 $\pm$ 0.249 \\
\textbf{KernelCF (lin)} & 0.144 $\pm$ 0.261 & 0.248 $\pm$ 0.284 & 0.352 $\pm$ 0.150 & 0.291 $\pm$ 0.182 & 0.560 $\pm$ 0.216 & 0.289 $\pm$ 0.271 \\
\textbf{KernelCF (RBF)} & 0.166 $\pm$ 0.285 & 0.329 $\pm$ 0.390 & \underline{0.369 $\pm$ 0.176} & 0.293 $\pm$ 0.207 & 0.595 $\pm$ 0.221 & 0.326 $\pm$ 0.314 \\
\textbf{LODA} & 0.220 $\pm$ 0.262 & 0.252 $\pm$ 0.280 & 0.282 $\pm$ 0.170 & 0.276 $\pm$ 0.118 & 0.557 $\pm$ 0.227 & 0.301 $\pm$ 0.258 \\
\textbf{OCSVM} & 0.256 $\pm$ 0.288 & 0.355 $\pm$ 0.391 & 0.364 $\pm$ 0.157 & 0.265 $\pm$ 0.106 & 0.583 $\pm$ 0.220 & 0.354 $\pm$ 0.296 \\
\textbf{PCA} & 0.269 $\pm$ 0.274 & 0.368 $\pm$ 0.403 & 0.208 $\pm$ 0.205 & 0.265 $\pm$ 0.092 & 0.531 $\pm$ 0.256 & 0.333 $\pm$ 0.300 \\
\textbf{UCF} & \textbf{0.313 $\pm$ 0.252} & \textbf{0.437 $\pm$ 0.386} & 0.332 $\pm$ 0.137 & \textbf{0.364 $\pm$ 0.112} & \textbf{0.651 $\pm$ 0.235} & \textbf{0.414 $\pm$ 0.288} \\
\midrule
\textbf{Average} & 0.215 $\pm$ 0.228 & 0.311 $\pm$ 0.318 & 0.300 $\pm$ 0.155 & 0.267 $\pm$ 0.116 & 0.552 $\pm$ 0.226 & 0.316 $\pm$ 0.263 \\
\bottomrule
\end{tabular}
}%
\label{tab:ap_performances_summary_ar}
\end{table}

\begin{figure}[H]
    \centering
    \includegraphics[width=\linewidth]{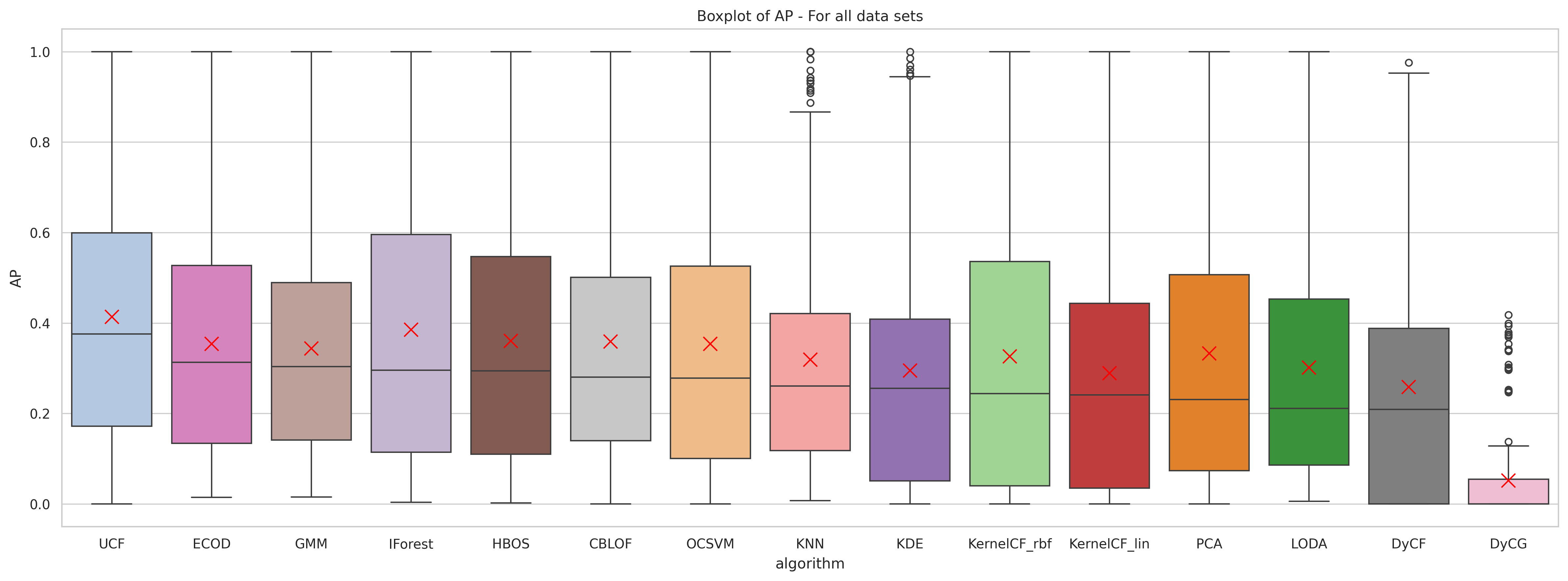}
    \caption{Boxplot of Average Precision (AP) for all 47 ADBench datasets}
    \vspace{-2em}
    \caption*{The red cross represents the mean.}
    \label{fig:boxplot_all_datasets}
\end{figure}

UCF consistently ranks among the top-performing methods across most dataset categories and anomaly ratios, outperforming Isolation Forest (IForest), the strongest baseline. Specifically, UCF achieves the highest average precision for medium, large, and high-dimensional datasets, with improvements of 10\%, 5\%, and 5\% with respect to the second-best in each category. For small datasets, it ranks fourth with performance approximately $2\%$ lower than the best method and $0.5\%$ lower than the second-best method. When categorized by anomaly ratio, UCF is the top-performing method in four out of five categories. Overall, UCF is best-performing method, demonstrating a 7\% improvement over IForest (the second-best overall) and a 13\% improvement over HBOS (the third-best overall).

Surprisingly, all algorithms, including UCF, exhibit improved performance as the anomaly ratio increases. This trend is counterintuitive, as unsupervised methods are typically sensitive to the presence of anomalies in the training set. This observation may indicate potential inconsistencies in dataset labeling, warranting further investigation.

Anomaly detection performance is highly dependent on the specific application (e.g., fraud detection, network intrusion, or failure prediction). While our results provide a general benchmark for the effectiveness of UCF, we emphasize that practitioners should validate methods on their own use cases to ensure suitability. However, the strong and consistent performance of UCF across diverse datasets emphasizes its potential as a robust anomaly detection approach.

\begin{table}[h]
\centering
\caption{Inference time (in seconds) per samples per datasets categories}
\vspace{-2em}
\caption*{\small Best result in \textbf{bold}, second-best \underline{underlined}.}
\resizebox{0.98\textwidth}{!}{%
\begin{tabular}{lccccc}
\toprule
\textbf{Algorithm} & \textbf{Small} & \textbf{Medium} & \textbf{Large} & \textbf{High-dimensional} & \textbf{All} \\
\midrule
\textbf{CBLOF} & 2.04e-05 $\pm$ 2.09e-05 & 8.36e-05 $\pm$ 3.70e-05 & 1.22e-05 $\pm$ 1.76e-05 & 4.47e-05 $\pm$ 2.33e-05 & \underline{4.24e-05 $\pm$ 3.93e-05} \\
\textbf{DyCF} & 5.75e-02 $\pm$ 5.52e-02$^*$ & 9.19e-03 $\pm$ 6.00e-03$^*$ & 2.90e-03 $\pm$ 5.78e-03$^*$ & 7.88e-02 $\pm$ 7.27e-02$^*$ & 2.58e-02 $\pm$ 4.17e-02$^*$ \\
\textbf{DyCG} & 1.65e-01 $\pm$ 2.31e-01$^*$ & 2.05e-02 $\pm$ 2.16e-02$^*$ & 1.20e-03 $\pm$ 1.29e-03$^*$ & $-$ & 6.23e-02 $\pm$ 1.48e-01$^*$ \\
\textbf{ECOD} & 1.06e-04 $\pm$ 2.09e-04 & 6.31e-05 $\pm$ 4.79e-05 & 4.55e-05 $\pm$ 5.89e-05 & 7.39e-04 $\pm$ 1.13e-03 & 1.99e-04 $\pm$ 5.51e-04 \\
\textbf{GMM} & 4.69e-06 $\pm$ 2.91e-06 & 1.93e-05 $\pm$ 3.26e-05 & 3.20e-06 $\pm$ 3.48e-06 & 5.20e-04 $\pm$ 1.16e-03 & 1.08e-04 $\pm$ 5.25e-04 \\
\textbf{HBOS} & \underline{3.98e-06 $\pm$ 2.17e-06} & \underline{4.34e-06 $\pm$ 6.90e-06} & \textbf{8.03e-07 $\pm$ 9.20e-07} & 3.39e-05 $\pm$ 4.13e-05 & \textbf{9.09e-06 $\pm$ 2.15e-05} \\
\textbf{IForest} & 5.16e-04 $\pm$ 5.07e-04 & 4.66e-05 $\pm$ 2.68e-05 & 1.42e-05 $\pm$ 7.99e-06 & 5.77e-05 $\pm$ 5.08e-05 & 1.61e-04 $\pm$ 3.26e-04 \\
\textbf{KDE} & 7.66e-05 $\pm$ 1.40e-04 & 1.11e-03 $\pm$ 7.62e-04 & 1.33e-02 $\pm$ 1.20e-02 & 6.33e-03 $\pm$ 4.98e-03$^*$ & 3.29e-03 $\pm$ 6.26e-03$^*$ \\
\textbf{KNN} & 8.37e-03 $\pm$ 1.29e-02 & 2.16e-03 $\pm$ 1.92e-03 & 3.97e-04 $\pm$ 5.81e-04 & 2.57e-03 $\pm$ 1.15e-03 & 3.41e-03 $\pm$ 7.11e-03 \\
\textbf{KernelCF (lin)} & 1.03e-02 $\pm$ 6.40e-03 & 1.68e-03 $\pm$ 7.78e-04 & 9.05e-02 $\pm$ 1.54e-01$^*$ & 1.49e-03 $\pm$ 4.79e-04$^*$ & 1.16e-02 $\pm$ 4.38e-02$^*$ \\
\textbf{KernelCF (RBF)} & 1.38e-02 $\pm$ 4.38e-03 & 3.16e-03 $\pm$ 1.58e-03 & 1.93e-03 $\pm$ 3.94e-04$^*$ & 3.31e-03 $\pm$ 7.98e-04$^*$ & 6.66e-03 $\pm$ 5.77e-03$^*$ \\
\textbf{LODA} & 1.96e-05 $\pm$ 8.67e-06 & 3.07e-05 $\pm$ 3.21e-05 & 7.22e-05 $\pm$ 9.43e-05 & 4.17e-03 $\pm$ 1.16e-02 & 8.31e-04 $\pm$ 5.13e-03 \\
\textbf{OCSVM} & 1.06e-05 $\pm$ 5.31e-06 & 2.50e-04 $\pm$ 1.37e-04 & 1.69e-03 $\pm$ 1.22e-03 & 3.00e-03 $\pm$ 5.72e-03$^*$ & 9.08e-04 $\pm$ 2.73e-03$^*$ \\
\textbf{PCA} & \textbf{3.88e-06 $\pm$ 2.16e-06} & \textbf{1.56e-06 $\pm$ 3.16e-06} & \underline{1.04e-06 $\pm$ 1.61e-06}$^*$ & 1.71e-03 $\pm$ 3.25e-03$^*$ & 1.77e-04 $\pm$ 1.05e-03$^*$ \\
\textbf{UCF} & 2.10e-02 $\pm$ 1.11e-02 & 3.81e-03 $\pm$ 1.73e-03 & 1.22e-02 $\pm$ 8.61e-03 & 7.75e-03 $\pm$ 8.77e-03 & 1.09e-02 $\pm$ 1.03e-02 \\
\midrule
\textbf{Average} & 1.85e-02 $\pm$ 2.15e-02 & 2.81e-03 $\pm$ 2.31e-03 & 8.28e-03 $\pm$ 1.23e-02 & $-$ & 8.43e-03 $\pm$ 1.82e-02 \\
\bottomrule
\end{tabular}
}%
\vspace{-0.5em}
\caption*{$^*$ \footnotesize{Did not work on all datasets of this category. The mean may not be representative.}}
\label{tab:it_performances_summary}
\end{table}

Furthermore, despite the computational challenges posed by UCF on large datasets, \cref{tab:it_performances_summary} demonstrates that, with appropriate code optimizations, these difficulties can be mitigated, and the outlier score for a sample can be obtained almost instantaneously, regardless of the size of the training dataset. Additionally, unlike most other methods, UCF does not require any training or parameter tuning. This unique characteristic allows for the direct computation of an outlier score for any given sample without prior preparation.

\section{Conclusion}
\label{sec:conclusion}
In this paper, we addressed the dimensionality limitations of Christoffel function (CF) based anomaly detection by introducing the Univariate Christoffel Function (UCF), a formulation that maps each query point from $\mathbb{R}^d$ to $\mathbb{R}$ via squared distances to support points. This induces a pushforward measure enabling both anomaly detection and support density estimation, while preserving the original CF key theoretical properties: the on-off support dichotomy and accurate support shape capture.

Extensive experiments on the ADBench benchmark demonstrate that UCF consistently outperforms 14 state-of-the-art baselines in Average Precision, ranking first overall, with a 4\% improvement over Isolation Forest and over 10\% over HBOS. Its performance is particularly strong for medium- and high-dimensional datasets, where traditional CF-based methods are computationally intractable. From the computational point of view, it is noticeable that the UCF approach does not require neither training nor parameter tuning. Code optimization allows for an efficient inference computation.

By eliminating the exponential matrix inversion bottleneck, UCF expands the practical applicability of CF based methods to high-dimensional datasets while retaining their mathematical rigor and computational frugality. This work thus equips the community with a robust, explainable, and universally applicable anomaly detection method, particularly valuable in domains where interpretability and efficiency are critical. Future work may explore leveraging the specific theoretical properties of the UCF to enhance interpretability. The possibility to estimate the density also opens perspectives to confirm anomalies. Domain-specific adaptations may also be considered, but UCF’s consistent performance across diverse datasets already establishes it as a compelling addition to the anomaly detection toolkit.

\section*{Acknowledgments}
Florian Grivet acknowledges financial support from the Centre national d’études spatiales (CNES), France (ROR: \url{https://ror.org/04h1h0y33}).

This work has benefited from the AI Interdisciplinary Institute ANITI funded by the France 2030 program under the Grant agreements n°ANR-19-P3IA-0004 and n°ANR-23-IACL-0002.

During the preparation of this work, the authors used LLM in order to check the text grammar and style. After using this tool, the authors reviewed and edited the content as needed and hence take full responsibility for the content of the submitted article.

\appendix
\section{Dataset properties}
\label{app:dataset_properties}

The following table summarizes each dataset, including the number of samples (\# Samples), features (\# Features), anomalies (\# Anomaly), anomaly ratio (\% Anomaly), and its two categories: one based on sample size/number of features (Category), and the other based on its anomaly ratio (Anomaly Ratio), as described in \cref{sec:datasets_algorithms}.

\begin{table}[httb]
\centering
\caption{Properties of the 47 datasets in ADBench}
\resizebox{0.98\textwidth}{!}{%
\begin{tabular}{lcccc|cc}
\toprule
\textbf{Data set} & \textbf{\# Samples} & \textbf{\# Features} & \textbf{\# Anomaly} & \textbf{\% Anomaly} & \textbf{Category} & \textbf{Anomaly Ratio} \\
\midrule
\textbf{aloi} & 49534 & 27 & 1508 & 3.04 & Large & 3\% - 8\% \\
\textbf{annthyroid} & 7200 & 6 & 534 & 7.42 & Medium & 3\% - 8\% \\
\textbf{backdoor} & 95329 & 196 & 2329 & 2.44 & High-dimensional & $<$ 3\% \\
\textbf{breastw} & 683 & 9 & 239 & 34.99 & Small & $>$ 25\% \\
\textbf{campaign} & 41188 & 62 & 4640 & 11.27 & High-dimensional & 8\% - 13\% \\
\textbf{cardio} & 1831 & 21 & 176 & 9.61 & Medium & 8\% - 13\% \\
\textbf{cardiotocography} & 2114 & 21 & 466 & 22.04 & Medium & 13\% - 25\% \\
\textbf{celeba} & 202599 & 39 & 4547 & 2.24 & Large & $<$ 3\% \\
\textbf{census} & 299285 & \underline{500} & 18568 & 6.2 & High-dimensional & 3\% - 8\% \\
\textbf{cover} & 286048 & 10 & 2747 & 0.96 & Large & $<$ 3\% \\
\textbf{donors} & \textbf{619326} & 10 & \underline{36710} & 5.93 & Large & 3\% - 8\% \\
\textbf{fault} & 1941 & 27 & 673 & 34.67 & Medium & $>$ 25\% \\
\textbf{fraud} & 284807 & 29 & 492 & 0.17 & Large & $<$ 3\% \\
\textbf{glass} & 214 & 7 & 9 & 4.21 & Small & 3\% - 8\% \\
\textbf{hepatitis} & 80 & 19 & 13 & 16.25 & Small & 13\% - 25\% \\
\textbf{http} & \underline{567498} & 3 & 2211 & 0.39 & Large & $<$ 3\% \\
\textbf{internetads} & 1966 & \textbf{1555} & 368 & 18.72 & High-dimensional & 13\% - 25\% \\
\textbf{ionosphere} & 351 & 32 & 126 & \underline{35.9} & Small & $>$ 25\% \\
\textbf{landsat} & 6435 & 36 & 1333 & 20.71 & Medium & 13\% - 25\% \\
\textbf{letter} & 1600 & 32 & 100 & 6.25 & Medium & 3\% - 8\% \\
\textbf{lymphography} & 148 & 18 & 6 & 4.05 & Small & 3\% - 8\% \\
\textbf{magic.gamma} & 19020 & 10 & 6688 & 35.16 & Large & $>$ 25\% \\
\textbf{mammography} & 11183 & 6 & 260 & 2.32 & Large & $<$ 3\% \\
\textbf{mnist} & 7603 & 100 & 700 & 9.21 & High-dimensional & 8\% - 13\% \\
\textbf{musk} & 3062 & 166 & 97 & 3.17 & High-dimensional & 3\% - 8\% \\
\textbf{optdigits} & 5216 & 64 & 150 & 2.88 & High-dimensional & $<$ 3\% \\
\textbf{pageblocks} & 5393 & 10 & 510 & 9.46 & Medium & 8\% - 13\% \\
\textbf{pendigits} & 6870 & 16 & 156 & 2.27 & Medium & $<$ 3\% \\
\textbf{pima} & 768 & 8 & 268 & 34.9 & Small & $>$ 25\% \\
\textbf{satellite} & 6435 & 36 & 2036 & 31.64 & Medium & $>$ 25\% \\
\textbf{satimage-2} & 5803 & 36 & 71 & 1.22 & Medium & $<$ 3\% \\
\textbf{shuttle} & 49097 & 9 & 3511 & 7.15 & Large & 3\% - 8\% \\
\textbf{skin} & 245057 & 3 & \textbf{50859} & 20.75 & Large & 13\% - 25\% \\
\textbf{smtp} & 95156 & 3 & 30 & 0.03 & Large & $<$ 3\% \\
\textbf{spambase} & 4207 & 57 & 1679 & \textbf{39.91} & High-dimensional & $>$ 25\% \\
\textbf{speech} & 3686 & 400 & 61 & 1.65 & High-dimensional & $<$ 3\% \\
\textbf{stamps} & 340 & 9 & 31 & 9.12 & Small & 8\% - 13\% \\
\textbf{thyroid} & 3772 & 6 & 93 & 2.47 & Medium & $<$ 3\% \\
\textbf{vertebral} & 240 & 6 & 30 & 12.5 & Small & 8\% - 13\% \\
\textbf{vowels} & 1456 & 12 & 50 & 3.43 & Medium & 3\% - 8\% \\
\textbf{waveform} & 3443 & 21 & 100 & 2.9 & Medium & $<$ 3\% \\
\textbf{wbc} & 223 & 9 & 10 & 4.48 & Small & 3\% - 8\% \\
\textbf{wdbc} & 367 & 30 & 10 & 2.72 & Small & $<$ 3\% \\
\textbf{wilt} & 4819 & 5 & 257 & 5.33 & Medium & 3\% - 8\% \\
\textbf{wine} & 129 & 13 & 10 & 7.75 & Small & 3\% - 8\% \\
\textbf{wpbc} & 198 & 33 & 47 & 23.74 & Small & 13\% - 25\% \\
\textbf{yeast} & 1484 & 8 & 507 & 34.16 & Medium & $>$ 25\% \\
\bottomrule
\end{tabular}
}%
\label{tab:datasets_properties}
\end{table}

\newpage
\bibliographystyle{siamplain}
\bibliography{references}

\begin{thebibliography}{10}

\bibitem{askari2018}
{\sc A.~Askari, F.~Yang, and L.~E. Ghaoui}, {\em Kernel-based outlier detection
  using the inverse christoffel function}, 2018,
  \url{https://arxiv.org/abs/1806.06775},
  \url{https://arxiv.org/abs/1806.06775}.

\bibitem{Baran}
{\sc M.~Baran}, {\em Complex equilibrium measure and bernstein type theorems
  for compact sets in $\mathbb{R}^n$}, Proc. Amer. Math. Soc., 123 (1995),
  pp.~485--494.

\bibitem{bedford}
{\sc E.~Bedford and B.~A. Taylor}, {\em The complex equilibrium measure of a
  symmetric convex set in $\mathbb{R}^n$}, Trans. Amer. Math. Soc., 294 (1986),
  pp.~705--717.

\bibitem{billet2026}
{\sc L.~Billet, L.~Travé-Massuyès, E.~Chanthery, and A.~Gaffet}, {\em Cloe:
  Christoffel loss autoencoder for anomaly detection}, in Submitted to The 16th
  International Conference on Information Science and Technology, 2026.

\bibitem{Chazal}
{\sc F.~Chazal, D.~Cohen-{S}teiner, and Q.~M\'erigot}, {\em Geometric inference
  for probability measures}, Found. Comp. Math., 11 (2011), pp.~733--751.

\bibitem{Ducharlet2025}
{\sc K.~Ducharlet, L.~Trav{\'e}-Massuy{\`e}s, J.-B. Lasserre, M.-V. Le~Lann,
  and Y.~Miloudi}, {\em Leveraging the christoffel function for outlier
  detection in data streams}, International Journal of Data Science and
  Analytics, 20 (2025), pp.~2021--2037,
  \url{https://doi.org/10.1007/s41060-024-00581-2},
  \url{https://doi.org/10.1007/s41060-024-00581-2}.

\bibitem{hbos}
{\sc M.~Goldstein and A.~Dengel}, {\em Histogram-based outlier score (hbos): A
  fast unsupervised anomaly detection algorithm}, KI-2012: poster and demo
  track, 1 (2012), pp.~59--63.

\bibitem{adbench}
{\sc S.~Han, X.~Hu, H.~Huang, M.~Jiang, and Y.~Zhao}, {\em Adbench: Anomaly
  detection benchmark}, in Neural Information Processing Systems (NeurIPS),
  2022.

\bibitem{kmeans}
{\sc D.~M. Hawkins}, {\em Identification of Outliers}, Springer Netherlands,
  Dordrecht, 1980, \url{https://doi.org/10.1007/978-94-015-3994-4},
  \url{http://link.springer.com/10.1007/978-94-015-3994-4}.

\bibitem{cblof}
{\sc Z.~He, X.~Xu, and S.~Deng}, {\em Discovering cluster-based local
  outliers}, Pattern Recognition Letters, 24 (2003), pp.~1641--1650,
  \url{https://doi.org/https://doi.org/10.1016/S0167-8655(03)00003-5},
  \url{https://www.sciencedirect.com/science/article/pii/S0167865503000035}.

\bibitem{klimek}
{\sc M.~Klimek}, {\em Pluripotentiual {T}heory}, Clarendon Press, UK, 1992.

\bibitem{kroo2013christoffel}
{\sc A.~Kro{\'o} and D.~Lubinsky}, {\em Christoffel functions and universality
  in the bulk for multivariate orthogonal polynomials}, Canadian Journal of
  Mathematics, 65 (2013), pp.~600--620.

\bibitem{Lass-Slot}
{\sc J.~Lasserre and L.~Slot}, {\em A {C}hristoffel-like function for
  high-dimensional support inference in graphical models}, J. Approx. Theory,
  (2026).
\newblock article number 106309.

\bibitem{sos}
{\sc J.-B. Lasserre and E.~Pauwels}, {\em Sorting out typicality with the
  inverse moment matrix sos polynomial}, in Proceedings of the 30th
  International Conference on Neural Information Processing Systems, NIPS'16,
  Red Hook, NY, USA, 2016, Curran Associates Inc., p.~190–198.

\bibitem{Lasserre2019}
{\sc J.~B. Lasserre and E.~Pauwels}, {\em The empirical {C}hristoffel function
  with applications in data analysis}, Adv. Comput. Math., 45 (2019),
  pp.~1439--1468, \url{https://doi.org/10.1007/s10444-019-09673-1},
  \url{https://doi.org/10.1007/s10444-019-09673-1}.

\bibitem{Lasserreetal2022}
{\sc J.~B. Lasserre, E.~Pauwels, and M.~Putinar}, {\em The
  Christoffel–Darboux Kernel for Data Analysis}, Cambridge University Press,
  2022.

\bibitem{liscalable}
{\sc Z.~Li, Q.~Huang, Y.~Zhu, L.~Yang, M.~M. Amiri, N.~van Stein, and M.~van
  Leeuwen}, {\em Scalable, explainable and provably robust anomaly detection
  with one-step flow matching}, in The Thirty-ninth Annual Conference on Neural
  Information Processing Systems (NeurIPS 2025), 2025.

\bibitem{ecod}
{\sc Z.~Li, Y.~Zhao, X.~Hu, N.~Botta, C.~Ionescu, and G.~H. Chen}, {\em Ecod:
  Unsupervised outlier detection using empirical cumulative distribution
  functions}, IEEE Transactions on Knowledge and Data Engineering, 35 (2023),
  pp.~12181--12193, \url{https://doi.org/10.1109/TKDE.2022.3159580}.

\bibitem{iforest}
{\sc F.~T. Liu, K.~M. Ting, and Z.-H. Zhou}, {\em Isolation forest}, in 2008
  Eighth IEEE International Conference on Data Mining, 2008, pp.~413--422,
  \url{https://doi.org/10.1109/ICDM.2008.17}.

\bibitem{high-dim-stat}
{\sc A.~Maleki, S.~Sen, S.~Balakrishnan, V.~Zuber, C.~Gao, R.~Dudeja,
  C.~Thrampoulidis, A.~Zang, W.~Su, J.~Kluzowski, P.-L. Lo, and A.~Shojaie},
  {\em High-{D}imensional {S}tatistics: {R}eflections on {P}rogress and {O}pen
  {P}roblems}, tech. report, 2026.
\newblock arXiv:2605.05076.

\bibitem{mate-nevai-totik}
{\sc A.~Mat\'{e}, P.~Nevai, and V.~Totik}, {\em Szeg\"{o}'s extremum problem on
  the unit circle}, Ann. Math., 134 (1991), pp.~433--453.

\bibitem{gmm}
{\sc G.~W. Milligan}, {\em An algorithm for generating artificial test
  clusters}, Psychometrika, 50 (1985), pp.~123--127,
  \url{https://doi.org/10.1007/BF02294153},
  \url{https://doi.org/10.1007/BF02294153}.

\bibitem{gini}
{\sc C.~Mussard, A.~Charpentier, and S.~Mussard}, {\em {KNN} and k-means in
  gini prametric spaces}, in {ECAI} 2025 - 28th European Conference on
  Artificial Intelligence, 25-30 October 2025, Bologna, Italy - Including 14th
  Conference on Prestigious Applications of Intelligent Systems {(PAIS} 2025),
  I.~Lynce, N.~Murano, M.~Vallati, S.~Villata, F.~Chesani, M.~Milano,
  A.~Omicini, and M.~Dastani, eds., Frontiers in Artificial Intelligence and
  Applications, {IOS} Press, 2025, pp.~2394--2401,
  \url{https://doi.org/10.3233/FAIA251085},
  \url{https://doi.org/10.3233/FAIA251085}.

\bibitem{kde}
{\sc E.~Parzen}, {\em {On Estimation of a Probability Density Function and
  Mode}}, The Annals of Mathematical Statistics, 33 (1962), pp.~1065 -- 1076,
  \url{https://doi.org/10.1214/aoms/1177704472},
  \url{https://doi.org/10.1214/aoms/1177704472}.

\bibitem{loda}
{\sc T.~Pevn{\'y}}, {\em Loda: Lightweight on-line detector of anomalies},
  Machine Learning, 102 (2016), pp.~275--304,
  \url{https://doi.org/10.1007/s10994-015-5521-0},
  \url{https://doi.org/10.1007/s10994-015-5521-0}.

\bibitem{knn}
{\sc S.~Ramaswamy, R.~Rastogi, and K.~Shim}, {\em Efficient algorithms for
  mining outliers from large data sets}, in Proceedings of the 2000 ACM SIGMOD
  International Conference on Management of Data, SIGMOD '00, New York, NY,
  USA, 2000, Association for Computing Machinery, p.~427–438,
  \url{https://doi.org/10.1145/342009.335437},
  \url{https://doi.org/10.1145/342009.335437}.

\bibitem{ocsvm}
{\sc B.~Sch\"{o}lkopf, R.~C. Williamson, A.~Smola, J.~Shawe-Taylor, and
  J.~Platt}, {\em Support vector method for novelty detection}, in Advances in
  Neural Information Processing Systems, S.~Solla, T.~Leen, and K.~M\"{u}ller,
  eds., vol.~12, MIT Press, 1999,
  \url{https://proceedings.neurips.cc/paper_files/paper/1999/file/8725fb777f25776ffa9076e44fcfd776-Paper.pdf}.

\bibitem{pca}
{\sc M.-L. Shyu, S.-C. Chen, K.~Sarinnapakorn, and L.~Chang}, {\em A novel
  anomaly detection scheme based on principal component classifier}, Scientific
  and technical aerospace reports, 45 (2007).

\bibitem{vu2020rateconvergencegeometricinference}
{\sc {Vu, Mai Trang}, {Bachoc, François}, and {Pauwels, Edouard}}, {\em Rate
  of convergence for geometric inference based on the empirical christoffel
  function}, ESAIM: PS, 26 (2022), pp.~171--207,
  \url{https://doi.org/10.1051/ps/2022003}.

\end{thebibliography}

\end{document}